\documentclass{article}

\usepackage[preprint,nonatbib]{neurips_2020} % to post to Arxiv

\usepackage[utf8]{inputenc} % allow utf-8 input
\usepackage[T1]{fontenc}    % use 8-bit T1 fonts
\usepackage{amsmath,amssymb,amsfonts,amsthm}
\usepackage{mathtools}
\usepackage{graphicx}
\usepackage{textcomp}
\usepackage{xcolor}
\usepackage{mathrsfs} % for \mathscr{}
\usepackage[hidelinks]{hyperref} % for URL in reference
\usepackage{subcaption} % for \begin{subfigure}
\usepackage{wrapfig} % for \wrapfigure{}
\usepackage{caption}
\usepackage{booktabs}
\newif\ifhideproofs
\hideproofstrue
\def\BibTeX{{\rm B\kern-.05em{\sc i\kern-.025em b}\kern-.08em
    T\kern-.1667em\lower.7ex\hbox{E}\kern-.125emX}}
    
\def\<#1>{\boldsymbol{\mathbf{#1}}} % make characters bold
\providecommand{\norm}[1]{\lVert#1\rVert}
\providecommand{\abs}[1]{\lvert#1\rvert}
\DeclareMathOperator*{\argmax}{arg\,max}
\DeclareMathOperator{\sgn}{sgn} % from Appendix

\newtheorem{prp}{Proposition}
\newtheorem{thm}{Theorem}
\newtheorem{lem}{Lemma}

\newtheorem{innercustomprp}{Proposition}
\newenvironment{customprp}[1]
  {\renewcommand\theinnercustomprp{#1}\innercustomprp}
  {\endinnercustomprp}

\newtheorem{innercustomthm}{Theorem}
\newenvironment{customthm}[1]
  {\renewcommand\theinnercustomthm{#1}\innercustomthm}
  {\endinnercustomthm}

\newtheorem{innercustomlem}{Lemma}
\newenvironment{customlem}[1]
  {\renewcommand\theinnercustomlem{#1}\innercustomlem}
  {\endinnercustomlem}

\title{Analytical bounds on the local Lipschitz constants of affine-ReLU functions}
\author{%
Trevor Avant\\
University of Washington\\
\texttt{avant@uw.edu} \\
\And
Kristi A. Morgansen\\
University of Washington\\
\texttt{morgansen@aa.washington.edu} \\
}

\begin{document}

\maketitle

\begin{abstract}
In this paper, we determine analytical bounds on the local Lipschitz constants of of affine functions composed with rectified linear units (ReLUs).
Affine-ReLU functions represent a widely used layer in deep neural networks, due to the fact that convolution, fully-connected, and normalization functions are all affine, and are often followed by a ReLU activation function.
Using an analytical approach, we mathematically determine upper bounds on the local Lipschitz constant of an affine-ReLU function, show how these bounds can be combined to determine a bound on an entire network, and discuss how the bounds can be efficiently computed, even for larger layers and networks.
%Furthermore, the analytical nature of our approach helps to elucidate the roles of the bias, linear operator, and perturbation size on the sensitivity of these functions.
%We then show how our bounds can be combined to create a local Lipschitz bound for an entire network, and also how the bounds can be computed efficiently, even for larger layers and networks.
%The analytical nature of our approach also allow our bounds to be computed for high dimensional layers such as large convolution functions.
We show several examples by applying our results to AlexNet, as well as several smaller networks based on the MNIST and CIFAR-10 datasets. The results show that our method produces tighter bounds than the standard conservative bound (i.e. the product of the spectral norms of the layers' linear matrices), especially for small perturbations.
\end{abstract}

%%%%%%%%%%%%%%%%%%%%%%%%%%%%%%%%%%%%%%%%%%%%%%%%%%%%%%%%%%%%%%%%%%%%%%%%%%%%%%%%
\section{Introduction}

%Over the past 10 years, deep neural networks have revolutionized image classification and other image processing tasks, and have been applied to problems in many other areas..

%Although deep neural networks can achieve superior results in many tasks, these advancements have mostly been the result of heuristic approaches. As a result, deep neural networks currently lack a theoretical understanding, which makes it dangerous to apply them in safety-critical applications.

\subsection{Introduction}

The huge successes of deep neural networks have been accompanied by the unfavorable property of high sensitivity. As a result, for many networks a small perturbation of the input can produce a huge change in the output \cite{Szegedy}.
%These examples were termed ``adversarial examples''.
These sensitivity properties are still not completely theoretically understood, and raise significant concerns when applying neural networks to safety-critical and other applications. This establishes a strong motivation to obtain a better theoretical understanding sensitivity.

One of the main tools in analyzing the sensitivity of neural networks is the Lipschitz constant, which is a measure of how much the output of a function can change with respect to changes in the input. Analytically computing the exact Lipschitz constant of neural networks has so far been unattainable due to the complexity and high-dimensionality of most networks. As feedforward neural networks consist of a composition of functions, a conservative upper bound on the Lipschitz constant can be determined by calculating the product of each individual function's Lipschitz constant \cite{Szegedy}. Unfortunately, this method usually results in a very conservative bound.

Calculating or estimating tighter bounds on Lipschitz constants has recently been approached using optimization-based methods \cite{Scaman,Latorre,Fazlyab,Zou,Jordan}. The downside to these approaches is that they usually can only be applied to small networks, and also often have to be relaxed, which invalidates any guarantee on the bound.

In summary, the current state of Lipschitz analysis of neural networks is that the function-by-function approach yields bounds which are too loose, and holistic approaches are too expensive for larger networks. In this paper, we explore a middle ground between these two approaches by analyzing the composite of two functions, the affine-ReLU function, which represents a common layer used in modern neural networks. This function is simple enough to obtain analytical results, but complex enough to provide tighter bounds than the function-by-function analysis. We can also combine the constants between layers to compute a Lipschitz constant for the entire network. Furthermore, our analytical approach leads us to develop intuition behind the structure of neural network layers, and shows how the different components of the layer (e.g. the linear operator, bias, nominal input, and size of the perturbation) contribute to sensitivity. 

\subsection{Related work}

The high sensitivity of deep neural networks has been noted as early as \cite{Szegedy}. This work conceived the idea of adversarial examples, which have since become a popular area of research \cite{Goodfellow2014}.
%linear regions: \cite{Pascanu}, 
One tool that has been used to study the sensitivity of networks is the input-output Jacobian \cite{Novak, Sokolic}, which gives a local estimate of sensitivity but generally provides no guarantees.

Lipschitz constants are also a common tool to study sensitivity. Recently, several studies have explored using optimization-based approaches to compute the Lipschitz constant of neural networks. The work of \cite{Scaman} presents two algorithms, AutoLip and SeqLip, to compute the Lipschitz constant. AutoLip reduces to the standard conservative approach, while SeqLip is an optimization which requires a greedy approximation for larger networks. The work of \cite{Latorre} presents a sparse polynomial optimization (LiPopt) method which relies on the network being sparse. To apply this technique to larger networks, the authors have to first prune the network to increase sparsity. A semidefinite programming technique (LipSDP) is used in \cite{Fazlyab}, but in order to apply the results to larger networks, a relaxation must be used which invalidates the guarantee. Another approach is that of \cite{Zou}, in which linear programming is used to estimate Lipschitz constants. Finally, \cite{Jordan} proposes exactly computing the Lipschitz constant using mixed integer programming, which is very expensive and can only be applied to very small networks.

\iffalse
Optimization-based approaches to compute the Lipschitz constant of neural networks are currently a popular area of research. As a summary of theses methods, algorithmic approaches (AutoLip and SeqLip) in \cite{Scaman}, sparse polynomial optimization (LiPopt) in \cite{Latorre}, semidefinite programming in \cite{Fazlyab}, linear programming in \cite{Zou}, and mixed integer programming in \cite{Jordan}. For all of these methods are only able to be applied to small networks, and to apply them to larger networks requires relaxations which invalidate guarantees.
\fi
Other work has considered computing Lipschitz constants in the context of adversarial examples \cite{Peck, Tsuzuku, Weng}. While these works use Lipschitz constants and similar mathematical analysis, their focus is on classification, and it is not clear how or if these techniques can be adapted to provide guaranteed upper Lipschitz bounds for larger networks. Additionally, other work has considered constraining the Lipschitz constant as a means to regularize a network \cite{Gouk,Terjek,Bartlett}. Finally, we note that in this work we study affine-ReLU functions which are commonly used in neural networks, but we are not aware of any work that has directly analyzed these functions except for \cite{Dittmer}.

%The work of \cite{Weng} uses an extremal value approach which is a statistical estimation. The work of \cite{Peck} presents lower bounds on adversarial examples. The work of \cite{Tsuzuku} presents a lower bound on adversarial examples.

%Weng: compute samples, compute max norm of gradient of samples, use extremal value theorem. CLEVER is a statistical estimate. Adversarial methods, however use similar tools such as Lipschitz constants.
%Recently, attention has been given to the spectral norms of convolution functions \cite{Sokolic}. It has been shown that !!!. One of the main ways in which neural networks have been ... Other work:  deals with classification, we are concerned with a mathematical input-output sensitivity, and are not specifically considering classification networks. 

%Our approach is based on the fact that many of the most common neural network layers, such as convolution and fully-connected layers, are affine transformations. Additionally, these functions are followed by a ReLU function, which makes the composite function an affine-ReLU composition. Thus, affine-ReLU functions describe many of the most common building blocks of neural networks. Analyzing the structure of affine-ReLU functions can yield tighter sensitivity bounds, as well as better insight into the nature of modern deep neural networks.

% Weng: see https://openreview.net/forum?id=BkUHlMZ0b

%The remainder of this paper is organized as follows....

\subsection{Contributions}

Our main contributions are that develop analytical upper bounds on the local Lipschitz constant of affine-ReLU function. We show how these bounds can be combined to create a bound on an entire feedforward network, and also how we can compute our bounds even for large layers.

\subsection{Notation}

In this paper, we use non-bold lowercase and capital ($a$ and $A$) to denote scalars, bold lowercase ($\<a>$) to denote vectors, and bold uppercase ($\<A>$) to denote matrices. Similarly, we use non-bold to denote scalar-valued functions ($f(\cdot)$) and bold to denote vector-valued functions ($\<f>(\cdot)$). We will use inequalities to compare vectors, and say that $\<a> > \<b>$ holds if all corresponding pairs of elements satisfy the inequality. Additionally, unless otherwise specified, we let $\norm{\cdot}$ denote the 2-norm.

%%%%%%%%%%%%%%%%%%%%%%%%%%%%%%%%%%%%%%%%%%%%%%%%%%%%%%%%%%%%%%%%%%%%%%%%%%%%%%%%
\section{Affine \& ReLU functions}

\subsection{Affine functions} \label{sec:affine_functions}

Affine functions are ubiquitous in deep neural networks, as convolutional, fully-connected, and normalization functions are all affine.
%Additionally, since the composition of affine functions is affine, the composition of multiple functions (e.g. normalization and convolution) may be affine as well.
%For example, in a convolution $\rightarrow$ batch normalization $\rightarrow$ ReLU layer, the convolution $\rightarrow$ batch normalization part is affine.
%A fully-connected layer is quite simply a linear operation which can be written as matrix. When this layer is created with a bias, it becomes affine. The convolution operator can be thought of as a tensor dot product between tensors (often called ``kernels'') and a portion of the input tensor. The kernel is applied to different regions of the input tensor, and a dot product is taken each time. The result of each dot product becomes one element of the multidimensional output tensor. Convolution is a linear relationship, and therefore, after reshaping the input and output tensors into 1D vectors, the convolution can be represented as a matrix.
%Note that while convolution is an affine function, the matrix that represents the linear part of the affine transformation is often too large to efficiently work with. Analysis of this matrix is made more tractable by realizing that it has a certain structure (sometimes called ``doubly block circulant'' \cite[Ch.~9]{Goodfellow}). Recent work has utilized this structure to make analysis more tractable, which has led to, for example, more efficient algorithms to compute the spectral norm \cite{Sedghi}.
An affine function can be written as $\<f>(\<x>) = \<A> \<x> + \<b>$ where  $\<A> \in \mathbb{R}^{m \times n}$, $\<x> \in \mathbb{R}^n$, and $\<b> \in \mathbb{R}^m$. Note that since we are considering neural networks, without loss of generality we can define $\<x>$ to be a tensor that has been reshaped into a 1D vector. We note that we can redefine the origin of the domain of an affine function to correspond to any point $\<x>_0$. More specifically, if we consider the system $\<A> \<x>{+}\<b>_0$ and nominal input $\<x>_0$, we can redefine $\<x>$ as $\<x>_0{+}\<x>$ in which case the affine function becomes $\<A> \<x>{+}\<A> \<x>_0{+}\<b>_0$. We then redefine the bias as $\<b> \coloneqq \<b>_0{+}\<A> \<x>_0$, which gives us the affine function $\<A> \<x>{+}\<b>$. In this paper, we will use the form $\<A> \<x>{+}\<b>$ to represent an affine function that has been shifted so that the origin is at $\<x>_0$, and $\<x>$ represents a perturbation from $\<x>_0$.

\subsection{Rectified Linear Units (ReLUs)}

The rectified linear unit (ReLU) is widely used as an activation function in deep neural networks. The ReLU is simply the elementwise maximum of an input and zero: $\<relu>(\<y>) = \max(\<0>,\<y>)$.  The ReLU is piecewise linear, and can therefore be represented by a piecewise linear matrix. To define this matrix, we will first define the following function which indicates if a value is non-negative: $\text{ind}(y) = \{1 ~\text{if}~ y<0, ~~ 0 ~\text{if}~ y\leq 0\}$. We will define the elementwise version of this function as $\<ind>: \mathbb{R}^m \rightarrow \{0,1\}^m$.
%So, for example, if $\<y> = [2 ~~ {-3}]^T$, then $\<ind>(\<y>)= [1 ~~ 0]^T$.
By defining $\<diag>: \mathbb{R}^m \rightarrow \mathbb{R}^{m \times m}$ as the function which creates a diagonal matrix from a vector, we define the ReLU matrix $\<R>_{\<y>}$ and ReLU function $\<relu>$ as follows
\begin{equation}
\<R>_{\<y>} = \<diag>(\<ind>(\<y>)) , ~~~~~~~~~~~~~~
\<relu>(\<y>) = \<R>_{\<y>} \<y> .
\label{eq:relu_matrix}
\end{equation}
Note that $\<R>_{\<y>}$ is a function of $\<y>$, but to make our notation clear we choose to denote the dependence on $\<y>$ using a subscript rather than parentheses.

Note that ReLUs are naturally related to the geometric concept of orthants, which are a higher dimensional generalization of quadrants in $\mathbb{R}^2$ and octants in $\mathbb{R}^3$. The $m$-dimensional space $\mathbb{R}^m$ has $2^m$ orthants.
%$\mathbb{R}^2$ has four quadrants, $\mathbb{R}^3$ has eight octants, etc.
Furthermore, the matrix $\<R>$ can be interpreted as an orthogonal projection matrix, which projects $\<y>$ onto a lower-dimensional space of $\mathbb{R}^m$. In $\mathbb{R}^3$ for example, $\<R>$ will represent a projection onto either the origin (when $\<R>=\<0>$), a coordinate axis, a coordinate plane, or all of $\mathbb{R}^3$ (when $\<R>=\<I>$). Each orthant in $\mathbb{R}^m$ corresponds to a linear region of the ReLU function, so since there are $2^m$ orthants there are $2^m$ linear regions.

\subsection{Affine-ReLU Functions} \label{sec:aff_relu_functions}

\begin{figure}[ht]
%\begin{wrapfigure}{r}{.7\textwidth}
\begin{center}
\includegraphics[width=.80\textwidth]{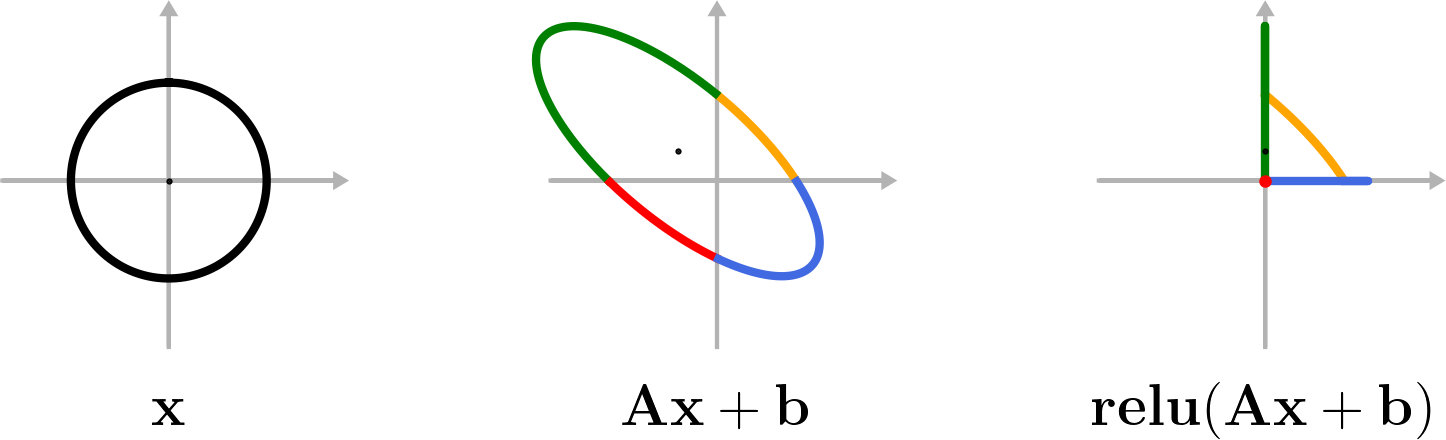}
\end{center}
\caption{The unit ball transformed by affine and ReLU functions in $\mathbb{R}^2$. Different colors represent different orthants after the affine operator. As shown in the rightmost diagram, the ReLU projects the domain onto the non-negative orthant.}
\label{fig:transformation}
%\end{wrapfigure}
\end{figure}

We define an affine-ReLU function as a ReLU composed with an affine function. In a neural network, these represent one layer, e.g. a convolution or fully-connected function with a ReLU activation. Using the notation in \eqref{eq:relu_matrix}, we can write an affine-ReLU function as
\begin{equation}
\<relu>(\<A> \<x> + \<b>) = \<R>_{\<A> \<x> + \<b>} (\<A> \<x> + \<b>) .
\label{eq:aff_relu_basic}
\end{equation}
\begin{figure}[ht]
\centering
\includegraphics[width=.80\textwidth]{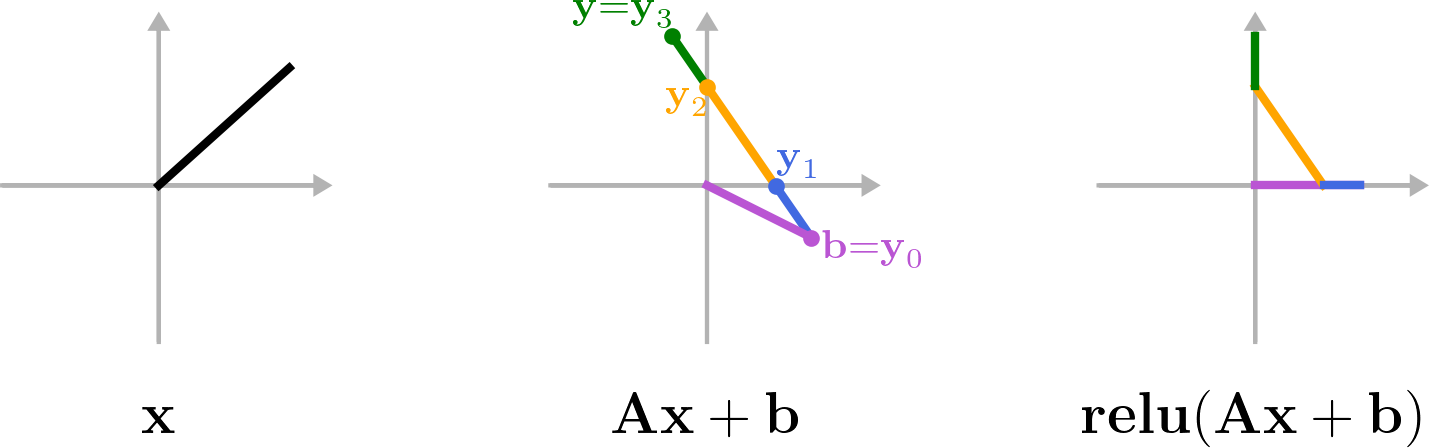}
\caption{A line segment transformed by affine and ReLU functions in $\mathbb{R}^2$.}
\label{fig:line_transformation}
\end{figure}
Using Fig. \ref{fig:line_transformation} as a reference, we note that the vector $\<y> = \<A> \<x> + \<b>$ will lie in several different orthants of $\mathbb{R}^m$, which are the linear regions of the ReLU function. As a result, $\<relu>(\<A> \<x> + \<b>)$ can be represented as a sum across the linear regions (i.e. orthants) of the function. For each $\<x>$, there will be some number $p$ of linear regions, and we define the points at which $\<y>$ transitions between linear regions as $\<y>_i = \alpha_i \<A> \<x> + \<b>$ for $i=1,...,p-1$ and where $0 < \alpha_i < 1$ and $\alpha_i > \alpha_{i-1}$. We also define $\alpha_0 = 0$ and $\alpha_p = 1$ so that $\<y>_0 = \<b>$ and $\<y>_p = \<y>$. The transition vectors $\<y>_i$ can be determined for a given $\<x>$ by determining the value of $\alpha_i$ for which elements of $\<y>_i = \alpha_i \<A> \<x> + \<b>$ equal zero. With these vectors defined, we note that any vectors $\<y>_{i-1}$ and $\<y>_i$ will lie in or at the boundary of the same orthant, and we define $\<R>_i$ as the ReLU matrix corresponding to that orthant. Therefore, we can write the net change of the affine-ReLU function across the orthant adjacent to $\<y>_{i-1}$ and $\<y>_i$ as 
\begin{align}
\<R>_i(\<y>_i - \<y>_{i-1}) = \<R>_i(\alpha_i \<A> \<x> + \<b> - (\alpha_{i-1} \<A> \<x> + \<b>)) = (\alpha_i - \alpha_{i-1}) \<R>_i \<A> \<x> .
\end{align}
Next, we define $\Delta \alpha_i \coloneqq \alpha_i - \alpha_{i-1}$ and note that $\sum_{i=1}^p \Delta \alpha_i = 1$. Noting that the change from the $\<0>$ to $\<b>$ segment of $\<y>$ is $\<R>_{\<b>} (\<b> - \<0>) = \<R>_{\<b>} \<b>$ (see Fig. \ref{fig:line_transformation}), we can write the affine-ReLU function as
\begin{align}
\<relu>(\<A> \<x> + \<b>)
&= \<R>_{\<b>} \<b> + \sum_{i=1}^p \Delta \alpha_i \<R>_i \<A> \<x> .
\label{eq:aff_relu_as_sum}
\end{align}

%%%%%%%%%%%%%%%%%%%%%%%%%%%%%%%%%%%%%%%%%%%%%%%%%%%%%%%%%%%%%%%%%%%%%%%%%%%%%%%%
\section{Lipschitz Constants}

\subsection{Local Lipschitz Constant}

We will analyze the sensitivity of the affine-ReLU functions using a Lipschitz constant, which measures how much the output of a function can change with respect to changes in the input. The Lipschitz constant of a function $\<f>: \mathbb{R}^n \rightarrow \mathbb{R}^m$ is $L = \sup_{\<x>_0 \neq \<x>_1} \norm{\<f>(\<x>_1) - \<f>(\<x>_0)}/\norm{\<x>_1 - \<x>_0}$.
 %$L = \sup_{\<x>_0, \<x>_1 \in \mathcal{X}} \norm{\<f>(\<x>_1) - \<f>(\<x>_0)}/\norm{\<x>_1 - \<x>_0}$
\iffalse
\begin{align}
L = \sup_{\<x>_0, \<x>_1 \in \mathcal{X}} \frac{\norm{\<f>(\<x>_1) - \<f>(\<x>_0)}}{\norm{\<x>_1 - \<x>_0}}
\label{eq:lipschitz_constant}
\end{align}
where $\mathcal{X} \subseteq \mathbb{R}^n$ is the domain of $\<f>$.
\fi
Our goal is to analyze the sensitivity of the affine-ReLU function by considering a nominal input and perturbation. As mentioned in Section \ref{sec:affine_functions}, we consider the affine function $\<A> \<x> + \<b>$ to be shifted so that the origin $\<x> = \<0>$ corresponds to the nominal input $\<x>_0$, and $\<x>$ corresponds to a perturbation. We define the local Lipschitz constant as a modified version of the standard Lipschitz constant:
\begin{align}
L( \<x>_0, \mathcal{X} ) &\coloneqq \max_{\<x> \in \mathcal{X}} \frac{\norm{\<f>(\<x>) - \<f>(\<0>)}}{\norm{\<x>}} .
\label{eq:local_lipschitz_constant}
\end{align}
The set $\mathcal{X} \subseteq \mathbb{R}^n$ represents the set of all permissible perturbations. In this paper, we will be most interested in the case that $\mathcal{X}$ is the Euclidean ball ($\mathcal{X} = \{ \<x> ~~ | ~~ \norm{\<x>} \leq \epsilon \}$) or the positive part of the Euclidean ball ($\mathcal{X} = \{ \<x> ~~ | ~~ \norm{\<x>} \leq \epsilon, ~~ \<x> \geq \<0> \}$). See Appendix \ref{sec:appendix_bounding_region} for more information. As $\mathcal{X}$ denotes the domain of affine function (and affine-ReLU function), we will define the range of the affine function similarly as
\begin{align}
\mathcal{Y} = \{ \<A> \<x> + \<b> ~~ | ~~ \<x> \in \mathcal{X} \} .
\end{align}

%Note that we have defined the metrics in the Lipschitz equation using the 2-norm, so the Lipschitz constant for an affine function $\<A> \<x> + \<b>$ is $\norm{\<A>}$, the spectral norm of $\<A>$. Also note that the ReLU is a non-expansive mapping, as it has a Lipschitz constant of 1 (since it is clear that $\norm{\<relu>(\<y>)} \leq \norm{\<y>}$).
Applying the local Lipschitz constant to the affine-ReLU function we have
\begin{align}
L \left( \<x>_0, \mathcal{X} \right)
&= \max_{\<x> \in \mathcal{X}} \frac{\norm{\<relu>(\<A> \<x> + \<b>) - \<relu>(\<b>)}}{\norm{\<x>}} .
\label{eq:aff_relu_lipschitz_constant}
\end{align}
\iffalse
or using \eqref{eq:aff_relu_as_sum} as
\begin{align}
L \left( \mathcal{X}, \<x>_0 \right)
&= \max_{\<x> \in \mathcal{X}} \frac{\norm{\sum_{i=1}^p \Delta \alpha_i \<R>_i \<A> \<x>}}{\norm{\<x>}} .
\label{eq:aff_relu_lipschitz_constant_sum}
\end{align}
\fi
Determining the Lipschitz constant above is a difficult problem due to the piecewise nature of the ReLU. We are not aware of a way to do this computation for high dimensional spaces, which prohibits us from exactly computing the Lipschitz constant. Instead, we will try to come up with a conservative bound. In this paper, we will present several bounds on the Lipschitz constant. The following lemma will serve as a starting point for several of our bounds.
\begin{lem} \label{lem:basic_bound}
Consider the affine function $\<A> \<x> + \<b>$, its domain $\mathcal{X}$, and the piecewise representation of the affine-ReLU function in \eqref{eq:aff_relu_as_sum}. We have the following upper bound on the affine-ReLU function's local Lipschitz constant: $L( \<x>_0, \mathcal{X} ) \leq \max_{\<x> \in \mathcal{X}} \sum_{i=1}^p \Delta \alpha_i \norm{\<R>_i \<A>}$.
\iffalse
\begin{align}
L( \<x>_0, \mathcal{X} ) \leq \max_{\<x> \in \mathcal{X}} \sum_{i=1}^p \Delta \alpha_i \norm{\<R>_i \<A>} .
\label{eq:basic_bound}
\end{align}
\fi
\end{lem}
\ifhideproofs
The proof is shown in Appendix \ref{sec:appendix_proofs}.
\else
\begin{proof}
We can start with \eqref{eq:aff_relu_lipschitz_constant} and plug in \eqref{eq:aff_relu_as_sum}:
\begin{align}
L \left( \mathcal{X}, \<x>_0 \right)
&= \max_{\<x> \in \mathcal{X}} \frac{\norm{\<relu>(\<A> \<x> + \<b>) - \<relu>(\<b>)}}{\norm{\<x>}} \\
&= \max_{\<x> \in \mathcal{X}} \frac{\norm{( \<R>_{\<b>} \<b> + \sum_{i=1}^p \Delta \alpha_i \<R>_i \<A> \<x> ) - \<R>_{\<b>} \<b>}}{\norm{\<x>}} \\
&= \max_{\<x> \in \mathcal{X}} \frac{\norm{\sum_{i=1}^p \Delta \alpha_i \<R>_i \<A> \<x>}}{\norm{\<x>}} \\
&\leq \max_{\<x> \in \mathcal{X}} \frac{\sum_{i=1}^p \Delta \alpha_i \norm{\<R>_i \<A> \<x>}}{\norm{\<x>}} \\
&\leq \max_{\<x> \in \mathcal{X}} \frac{\sum_{i=1}^p \Delta \alpha_i \norm{\<R>_i \<A>} \norm{\<x>}}{\norm{\<x>}} \\
&= \max_{\<x> \in \mathcal{X}} \sum_{i=1}^p \Delta \alpha_i \norm{\<R>_i \<A>} .
\end{align}
\end{proof}
\fi

\subsection{Naive and intractable upper bounds}

%We will now analyze the sensitivity of the affine-ReLU function $\<relu>(\<A> \<x> + \<b>)$. As we mentioned in Section \ref{sec:affine_functions}, the affine function $\<A> \<x> + \<b>$ is centered at $\<x>_0$ so $\<x>$ represents a perturbation from $\<x>_0$.

We will now approach the task of deriving an analytical upper bound on the local Lipschitz constant of the affine-ReLU function. We start by presenting a standard naive bound.
\begin{prp} \label{prp:naive_upper_bound}
Consider the affine function $\<A> \<x> + \<b>$ and its domain $\mathcal{X}$. The spectral norm of $\<A>$ is an upper bound on the affine-ReLU function's local Lipschitz constant: $L( \<x>_0, \mathcal{X} ) \leq \norm{\<A>}$.
\iffalse
\begin{align}
L( \<x>_0, \mathcal{X} ) \leq \norm{\<A>} .
\end{align}
\fi
\end{prp}
\ifhideproofs
The proof is shown in Appendix \ref{sec:appendix_proofs}.
\else
\begin{proof}
Consider the inequality from Lemma \ref{lem:basic_bound}. Note that the ReLU matrix $\<R>_i$ is a diagonal matrix of zeros and ones, so for any $\<v> \in \mathbb{R}^n$, the non-zero elements of $\<R>_i \<A> \<v>$ will be a subset of the non-zero elements of $\<A> \<v>$. Therefore, $\<A> \geq \norm{\<R>_i \<A>}$ for all $\<R>_i$ and $\<x> \in \mathcal{X}$, and since $\sum_{i=1}^p \Delta \alpha_i = 1$, we have $L( \<x>_0, \mathcal{X} ) \leq \norm{\<A>}$.
\end{proof}
\fi
This is a standard conservative bound that is often used in determining the Lipschitz constants of a full neural network. Next, we will attempt to create a tighter bound. Consider the term $\norm{\<R>_i \<A>}$ in the inequality in Lemma \ref{lem:basic_bound}. The ReLU matrices $\<R>_i$ are those that correspond to the vectors $\<y> \in \mathcal{Y}$. So if we can determine all possible ReLU matrices for the vectors in $\mathcal{Y}$, then we can determine a tighter bound on the Lipschitz constant. We start by defining the matrix $\<R>_{max}$ as
\begin{equation}
\<R>_{max} \coloneqq \{ \<R>_{\<y>} ~~ | ~~ \norm{\<R>_{\<y>} \<A>} \geq \norm{\<R>_{\<w>} \<A>}, ~~ \<y> \in \mathcal{Y}, ~~ \forall \<w> \in \mathcal{Y} \} . 
\label{eq:Rmax}
\end{equation}

\begin{prp} \label{prp:max_upper_bound}
Consider the affine function $\<A> \<x> + \<b>$, its domain $\mathcal{X}$, and matrix $\<R>_{max}$ defined in \eqref{eq:Rmax}. The spectral norm of $\<R>_{max} \<A>$ is an upper bound on the affine-ReLU function's local Lipschitz constant: $L( \<x>_0, \mathcal{X} ) \leq \norm{\<R>_{max} \<A>}$.
\iffalse
\begin{align}
L( \<x>_0, \mathcal{X} ) \leq \norm{\<R>_{max} \<A>} .
\end{align}
\fi
\end{prp}
\begin{proof}
Consider the inequality in Lemma \ref{lem:basic_bound}, and note that the ReLU matrices $\<R>_i$ correspond to vectors $\<y> \in \mathcal{Y}$. So, by definition of $\<R>_{max}$, we have $\norm{\<R>_{max} \<A>} \geq \norm{\<R>_i \<A>}$ for all $\<R>_i$ and all $\<x> \in \mathcal{X}$. Since $\sum_{i=1}^p \Delta \alpha_i = 1$, we have $L( \<x>_0, \mathcal{X} ) \leq \norm{\<R>_{max} \<A>}$.
\end{proof}
While this proposition would provide an upper bound on the Lipschitz constant, in practice it requires determining all possible ReLU matrices $\<R>_i$ corresponding to all vectors $\<y> \in \mathcal{Y}$. Since $\mathbb{R}^m$ has $2^m$ orthants, this method would most likely be intractable except for very small $m$ (due to the large number of matrices we would need to compare). As we do not know of a way that avoids computing a large number of spectral norms, this motivates us to look for an even more conservative bound that is more easily computable.

%%%%%%%%%%%%%%%%%%%%%%%%%%%%%%%%%%%%%%%%%%%%%%%%%%%%%%%%%%%%%%%%%%%%%%%%%%%%%%%%
\section{Upper bounding regions}

\begin{figure}[ht]
\centering
\includegraphics[width=.99\textwidth]{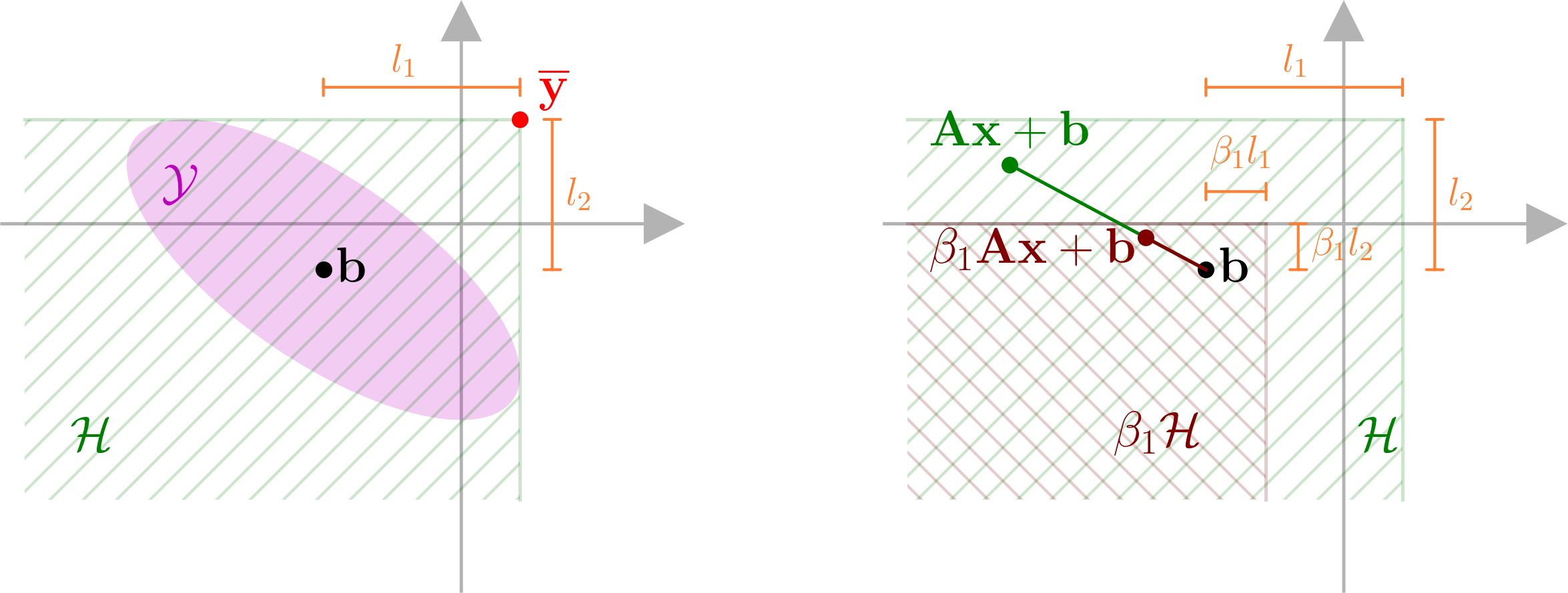}
%\includegraphics[width=.49\textwidth]{epsilons.png}
%\caption{Conservative upper bounds on the unit ball transformed by the affine transformation $\mathbf{A} \mathbf{x} + \mathbf{b}$, shown in $\mathbb{R}^2$. This diagram illustrates that part of the most conservative bound (green) may reside in an orthant that the original affine transformation (black) does not reside in.}
\caption{\textit{Left}: Diagram of the range $\mathcal{Y}$ of the affine function (in this case, a transformation of when the domain $\mathcal{X}$ is the Eucildean ball), its bounding region $\mathcal{H}$ with lengths $l_i$, and the upper bounding vertex $\overline{\<y>}$.
%This diagram illustrates that there may not be any part of $\mathcal{Y}$ that resides in the orthant that $\overline{\<y>}$ resides in, which is part of the reason our bound in conservative.
\textit{Right}: Diagram of a bounding region $\mathcal{H}$ and its scaled bounding region $\beta_1 \mathcal{H}$. Note that there will also be a larger bounding region $\beta_2 \mathcal{H}$ aligned with the vertical axis but it is not shown.}
\label{fig:bounding_region}
\end{figure}

\subsection{Bounding regions}

Our approach in determining a more easily computable bound is based on the idea that we can find a ReLU matrix $\overline{\<R>}$ such that $\norm{\overline{\<R>} \<A>} \geq \norm{\<R>_{\<y>} \<A>}$ for all $\<y> \in \mathcal{Y}$. To find this matrix, we will create a coordinate-axis-aligned bounding region around the set $\mathcal{Y}$ (see the left side of Fig. \ref{fig:bounding_region} for a diagram). We define the upper bounding vertex of this region as $\overline{\<y>}$, its associated ReLU matrix as $\overline{\<R>}$, and the upper bounding region as $\mathcal{H}$:
\begin{align}
\overline{\<y>} \coloneqq \{ \<y> ~~ | ~~ \<y> \geq \<A> \<x> + \<b>, ~~ \forall \<x> \in \mathcal{X} \}, ~~~~~~~
 \overline{\<R>} &\coloneqq \<R>_{\overline{\<y>}}, ~~~~~~~
\mathcal{H} \coloneqq \{ \<y> ~~ | ~~ \<y> \leq \overline{\<y>} \} .
\label{eq:y_bar_R_bar}
\end{align}
We define $\boldsymbol{l} \coloneqq \overline{\<y>} - \<b>$, which represents the distance in each coordinate direction from $\<b>$ to the border of the bounding region (see Fig. \ref{fig:bounding_region}). This region will function as a conservative estimate around $\mathcal{Y}$ in regards to the ReLU operation. For information about how $\overline{\<y>}$ and $\boldsymbol{l}$ can be computed for various domains $\mathcal{X}$, see Appendix \ref{sec:appendix_bounding_region}.

Recalling the definition of $\<R>$ in \eqref{eq:relu_matrix}, we note that $\<R>$ is a diagonal matrix of 0's and 1's. Since the matrix $\overline{\<R>}$ is computed with respect to $\overline{\<y>}$, it is a reflection of the ``most positive'' orthant in $\mathcal{H}$, and will have a 1 anywhere any matrix $\<R>_{\<y>}$ for $\<y> \in \mathcal{H}$ has a 1. The matrix $\overline{\<R>}$ can also be interpreted as the logical disjunction of all matrices $\<R>_{\<y>}$ for $\<y> \in \mathcal{H}$.
%which we can write as $\overline{\<R>} = \bigvee_i \<R>_{\<y>_i}, ~ \forall \<y>_i \in \mathcal{H}$.

\subsection{Nested bounding regions} \label{sec:nested_regions}

We have described the concept of an upper bounding region $\mathcal{H}$, which will lead us to develop an upper bound on the local Lipschitz constant. However, we will be able to develop an even tighter bound by noting that within a bounding region there will be some number of smaller, ``nested'' bounding regions, each with its own matrix $\overline{\<R>}$ (see the right side of Fig. \ref{fig:bounding_region}). We then note that the vector $\<y>$ can be described using a piecewise representation, in which pieces of $\<y>$ closer to $\<b>$ are contained in smaller bounding regions. We consider some number $q$ of the these bounding regions, which we will index by $i=1,...,q$. We define these bounding regions using scalars $0 \leq \beta_i \leq 1$ where $\beta_i > \beta_{i-1}$. For a given $\<x> \in \mathcal{X}$, we define the affine transformation of $\beta_i \<x>$ as $\<y>_i \coloneqq \beta_i \<A> \<x> + \<b>$. Furthermore, we define $\beta_0 = 0$ and $\beta_q = 1$ so that $\<y>_0 = \<b>$ and $\<y>_q = \<y>$. We define the scaled bounding region to be the region that bounds $\<y>_i$ for all $\<x> \in \mathcal{X}$. As in \eqref{eq:y_bar_R_bar}, we define the scaled bounding region as $\beta_i \mathcal{H}$, the bounding vertex as $\overline{\<y>}_i$, and its corresponding ReLU matrix as $\overline{\<R>}_i$:
\iffalse
\begin{align}
\beta_i \mathcal{H} \coloneqq \{ \<y> ~~ | ~~ \<y> \leq \<b> + \beta_i \<l> \} .
\label{eq:alpha_H}
\end{align}
\fi
\begin{align}
\overline{\<y>}_i \coloneqq \{ \<y> ~~ | ~~ \<y> \geq \beta_i \<A> \<x> + \<b>, ~~ \forall \<x> \in \mathcal{X} \}, ~~~~~~~
\overline{\<R>}_i &\coloneqq \<R>_{\overline{\<y>}_i}, ~~~~~~~
\beta_i \mathcal{H} \coloneqq \{ \<y> ~~ | ~~ \<y> \leq \overline{\<y>}_i \} .
\label{eq:beta_y_bar_R_bar}
\end{align}
Note that the distance from $\<b>$ to $\overline{\<y>}$ is $\boldsymbol{l}$ for $\mathcal{H}$, and the distance from $\<b>$ to $\overline{\<y>}_i$ is $\beta_i \boldsymbol{l}$ for $\beta_i \mathcal{H}$. It will be most sensible to define the scalars $\beta_i$ to occur at the points for which the scaled region enters positive space for each coordinate, which are the locations at which $\overline{\<R>}_i$ changes. These values can be found by determining when $\<b> + \beta_i \boldsymbol{l}$ equals zero for each coordinate. Also, we define the difference in $\beta_i$ values as $\Delta \beta_i \coloneqq \beta_i - \beta_{i-1}$ for $i=1,...,q$. Lastly, we define the following lemma which we will use later to create our bound.
\iffalse
The following lemma proves that the vector $\<y>_i$ is contained in the bounding region $\beta_i \mathcal{H}$.
\begin{lem}
Consider a domain $\mathcal{X}$, its range $\mathcal{Y}$ with bounding region $\mathcal{H}$, and scaled bounding region $\beta_i \mathcal{H}$. For any $\<x> \in \mathcal{X}$ and scalar $\beta_* \leq \beta_i$, the affine transformation of $\beta_* \<x>$ will lie in $\beta_i \mathcal{H}$.
\end{lem}
\begin{proof}
First, note that the vector $\<y>$ lies in $\mathcal{H}$ because the difference between $\<y>$ and $\<b>$ is $\<A> \<x>$, and $\<A> \<x> \leq \boldsymbol{l}$. Similarly, since $\beta_i \<A> \<x> \leq \beta_i \boldsymbol{l}$, the vector $\beta_i \<A> \<x> + \<b>$ lies in $\beta_i \mathcal{H}$. It follows that any vector $\beta_* \<A> \<x> + \<b>$ for $\beta_* < \beta_i$ lies in $\beta_i \mathcal{H}$ since $\beta_* \<A> \<x> \leq \beta_i \boldsymbol{l}$.
\end{proof}
\fi

\begin{lem}
Consider a bounding region $\mathcal{H}$ and its bounding ReLU matrix $\overline{\<R>}$. For any two points $\<y>_a, \<y>_b \in \mathcal{H}$, the following inequality holds: $\norm{\overline{\<R>} (\<y>_b - \<y>_a)} \geq \norm{\<R>_{\<y>_b} \<y>_b - \<R>_{\<y>_a} \<y>_a}$.
\iffalse
\begin{equation}
\norm{\overline{\<R>} (\<y>_b - \<y>_a)} \geq \norm{\<R>_{\<y>_b} \<y>_b - \<R>_{\<y>_a} \<y>_a} .
\label{eq:relu_matrix_diff_equation}
\end{equation}
\fi
\label{lem:bounding_region_vectors}
\end{lem}
\ifhideproofs
The proof is shown in Appendix \ref{sec:appendix_proofs}.
\fi

%%%%%%%%%%%%%%%%%%%%%%%%%%%%%%%%%%%%%%%%%%%%%%%%%%%%%%%%%%%%%%%%%%%%%%%%%%%%%%%%
\section{Upper bounds}

\subsection{Looser and tighter upper bounds}

We are now ready to present the main mathematical results of the paper.
\begin{thm} \label{thm:looser_bounds}
Consider the affine function $\<A> \<x> + \<b>$, its domain $\mathcal{X}$, and its bounding ReLU matrix $\overline{\<R>}$ from \eqref{eq:y_bar_R_bar}. The spectral norm of $\overline{\<R>} \<A>$ is an upper bound on the affine-ReLU function's local Lipschitz constant: $L( \<x>_0, \mathcal{X} ) \leq \norm{\overline{\<R>} \<A>}$.
\iffalse
\begin{align}
L( \<x>_0, \mathcal{X} ) \leq \norm{\overline{\<R>} \<A>} .
\end{align}
\fi
\end{thm}
\begin{proof}
From Lemma \ref{lem:basic_bound} we have $L( \<x>_0, \mathcal{X} ) \leq \max_{\<x> \in \mathcal{X}} \sum_{i=1}^p \Delta \alpha_i \norm{\<R>_i \<A>}$. We note that the matrix $\<R>_i$ corresponds to vectors which are inside the bounding region $\mathcal{H}$. Recalling that the ReLU matrices $\<R>$ are diagonal matrices with 0's and 1's on the diagonal, and $\overline{\<R>}$ will have a 1 anywhere any matrix $\<R>_{\<y>}$ for $\<y> \in \mathcal{H}$ has a 1. Therefore the non-zero elements of $\<R>_i \<w>$ will be a subset of the non-zero elements of $\overline{\<R>} \<w>$ for all $\<y> \in \mathcal{Y}$ and all $\<w> \in \mathbb{R}^m$, which implies $\norm{\overline{\<R>} \<A>} \geq \norm{\<R>_i \<A>}$ for all $i$ and all $\<x> \in \mathcal{X}$. Since $\sum_{i=1}^p \Delta \alpha_i = 1$, we have $L( \<x>_0, \mathcal{X} ) \leq \norm{\overline{\<R>} \<A>}$.
\end{proof}
\iffalse
\begin{proof}
Since $\<A> \<x> + \<b> \in \mathcal{H}, \forall \<x> \in \mathcal{X}$, using Lemma \ref{lem:bounding_region_vectors}, we can write the bound in \eqref{eq:aff_relu_senstiivity} as
\begin{align}
L( \<x>_0, \mathcal{X} ) &= \max_{\<x> \in \mathcal{X}} \frac{\norm{\<relu>(\<A> \<x> + \<b>) - \<relu>(\<b>)}}{\norm{\<x>}} \\
&= \max_{\<x> \in \mathcal{X}} \frac{\norm{\<R>_{\<A> \<x> + \<b>}(\<A> \<x> + \<b>) - \<R>_{\<b>} \<b>}}{\norm{\<x>}} \tag{Lem} \\
&\leq \max_{\<x> \in \mathcal{X}} \frac{\norm{\overline{\<R>}(\<A> \<x> + \<b> - \<b>)}}{\norm{\<x>}} \\
&= \max_{\<x> \in \mathcal{X}} \frac{\norm{\overline{\<R>} \<A> \<x>}}{\norm{\<x>}} \\
&\leq \max_{\<x> \in \mathcal{X}} \frac{\norm{\overline{\<R>} \<A>} \norm{\<x>}}{\norm{\<x>}} \\
&= \norm{\overline{\<R>} \<A>} .
\end{align}
\end{proof}
\fi

Computing this bound for a given domain $\mathcal{X}$ will be quick if we can quickly compute $\overline{\<R>}$ and $\norm{\overline{\<R>} \<A>}$. However, we can create an even tighter bound by using the idea of nested regions in Section \ref{sec:nested_regions}.

%\subsection{Tighter, more expensive upper bound}

\begin{thm} \label{thm:tighter_bounds}
Consider the affine function $\<A> \<x> + \<b>$ and its domain $\mathcal{X}$. Consider the nested bounding regions $\beta_i \mathcal{H}$, their scale factors $\Delta \beta_i$ and their bounding ReLU matrices $\overline{\<R>}_i$ as described in Section \ref{sec:nested_regions}. The following is an upper bound on the affine-ReLU function's local Lipschitz constant: $L( \<x>_0, \mathcal{X} ) \leq \sum_{i=1}^q \Delta \beta_i \norm{\overline{\<R>}_i \<A>}$.
\iffalse
\begin{equation}
L( \<x>_0, \mathcal{X} ) \leq \sum_{i=1}^q \Delta \beta_i \norm{\overline{\<R>}_i \<A>} .
\end{equation}
\fi
\end{thm}
\ifhideproofs
The proof is shown in Appendix \ref{sec:appendix_proofs}.
\fi
This is the last bound we will derive. To summarize the bounds from Proposition \ref{prp:naive_upper_bound} and Theorems \ref{thm:looser_bounds} \& \ref{thm:tighter_bounds}, we have
\begin{equation}
\norm{\<A>} \geq \norm{\overline{\<R>} \<A>} \geq \sum_{i=1}^q \Delta \beta_i \norm{\overline{\<R>}_i \<A>} \geq L( \<x>_0, \mathcal{X} ) .
\end{equation}
Note that the bound in Proposition \ref{prp:max_upper_bound} may be less than or greater than the bound in Theorem \ref{thm:tighter_bounds} so we cannot include it in the inequality above.

\subsection{Bounds for Multiple Layers} \label{sec:multiple_layers}

So far, our analysis has applied to a single affine-ReLU function, which would represents one layer (e.g. convolution-ReLU) of a network. We now describe how these bounds can be combined for multiple layers.
%We consider layer $i$, and let $\<x>^i$ denote the input to the $i^{th}$ layer, denote the output as $\<x>^{i+1} = \<A>^i \<x>^i + \<b>^i$. We can rewrite the Lipschitz bound in \eqref{eq:local_lipschitz_constant} as $L(\<x>_0, \mathcal{X}) \norm{\<x>} \geq \norm{\<f>(\<x>_0 + \<x>) - \<f>(\<x>_0)}$. Note that $\<x>^{(i+1)} = \<f>(\<x>^{(i)})$ so $\<x>_0^{(i+1)} = \<f>(\<x>_0^{(i)})$ and $\<x>_0^{(i+1)} = \<f>(\<x>_0^{(i)})$.
First, assume that we have a bound $\epsilon$ on the size of the perturbation $\<x>$, i.e. $\norm{\<x>} \leq \epsilon, ~~ \forall \<x> \in \mathcal{X}$. We can rearrange the local Lipschitz constant equation in \eqref{eq:local_lipschitz_constant} by moving the denominator to the LHS and applying the $\epsilon$ bound as follows
\begin{align}
%L( \<x>_0, \mathcal{X} ) &\geq \frac{\norm{\<f>(\<x>) - \<f>(\<0>)}}{\norm{\<x>}}, ~~ \forall \<x> \in \mathcal{X} \\
%L( \<x>_0, \mathcal{X} ) \norm{\<x>} &\geq \norm{\<f>(\<x>) - \<f>(\<0>)}, ~~ \forall \<x> \in \mathcal{X} \\
\epsilon L( \<x>_0, \mathcal{X} ) &\geq \norm{\<f>(\<x>) - \<f>(\<0>)}, ~~ \forall \<x> \in \mathcal{X} .
\end{align}
Recall that $\<f>(\<0>)$ represents the nominal input of the next layer, so $\<f>(\<x>) - \<f>(\<0>)$ represents the perturbation with respect to the next layer. Defining this perturbation as $\<z> \coloneqq \<f>(\<x>) - \<f>(\<0>)$, we have
\begin{align}
\epsilon L(\<x>_0, \mathcal{X}) \geq \norm{\<z>} .
\end{align}
This gives us a bound on perturbations of the nominal input of the next layer. We can therefore compute the local Lipschitz bounds in an iterative fashion, by propagating the perturbation bounds through each layer of the network. More specifically, if we start with $\epsilon$, we can compute the Lipschitz constant of the current layer and then determine the bound for the next layer. We can continue this process for subsequent layers. Using these perturbation bounds, we will consider the domains for each layer of the network to be of the form $\mathcal{X} = \{ \<x>  ~~ | ~~ \norm{\<x>} \leq \epsilon \}$ or $\mathcal{X} = \{ \<x>  ~~ | ~~ \norm{\<x>} \leq \epsilon, ~~ \<x> \geq \<0> \}$ when the layer is preceded by a ReLU. Also note that for other types of layers such as max pooling, if we can't compute a local bound, we can use the global bound, which is what we will do in our simulations.

%We also note that for a composition of affine-ReLU functions, a standard conservative bound on the Lipschitz constant could be computed as $L \geq \prod_i \norm{\<R>} \norm{\<A>^i}$. Using our bound in Theorem \ref{thm:looser_bounds}, we can compute the tighter bound $L \geq \prod_i \norm{\overline{\<R>} \<A>^i}$.

%%%%%%%%%%%%%%%%%%%%%%%%%%%%%%%%%%%%%%%%%%%%%%%%%%%%%%%%%%%%%%%%%%%%%%%%%%%%%%%%
\section{Simulation}

\subsection{Spectral norm computation}

Our results rely on computing the spectral norm $\<R> \<A>$ for various ReLU matrices $\<R>$. The $\<A>$ matrix will correspond to either a convolution or fully-connected function. For larger convolution functions, the $\<A>$ matrices are often too large to define explicitly. The only way we can compute $\norm{\<R> \<A>}$ for larger layers is by using a power iteration method.

To compute the spectral norm of a matrix $\<M>$, we can note that the largest singular value of $\<M>$ is the square root of the largest eigenvalue of $\<M>^T \<M>$. So, we can find the spectral norm of $\<M>$ by applying a power iteration to the operator $\<M>^T \<M>$. In our case, $\<M> = \<R> \<A>$ and $\<M>^T \<M> = \<A>^T \<R>^T \<R> \<A> = \<A>^T \<R> \<A>$ (for various $\<R>$ matrices). We can compute the operations corresponding to the $\<A>$, $\<A>^T$, and $\<R>$ matrices in code using convolution to apply $\<A>$, transposed convolution to apply $\<A>^T$, and zeroing appropriate elements to apply $\<R>$.
%This avoids ever having to explicitly represent $\<A>$ or $\<R>$.
%Furthermore, we can apply the power iteration in batches, where each sample corresponds to a particular $\<R>$ matrix.
In all of our simulations, we used 100 iterations, which we verified to be accurate for smaller systems for which an SVD can be computed for comparison.
%Note that this method is used in \cite{Scaman}, but only in regards to the matrix $\<A>$.

\subsection{Simulations}

In our simulations, we compared three different networks: a 7-layer network trained on MNIST, an 8-layer network trained on CIFAR-10, and AlexNet \cite{Krizhevsky} (11-layers, trained on ImageNet). See Appendix \ref{sec:appendix_architectures} for the exact architectures we used. We trained the MNIST and CIFAR-10 networks ourselves while we used the trained version of AlexNet from Pytorch's \texttt{torchvision} package. For all of our simulations, we used nominal input images which achieved good classification. However, we noticed that we obtained similar trends using random images. We compared the upper bounds of Proposition \ref{prp:naive_upper_bound}, Theorem \ref{thm:looser_bounds}, Theorem \ref{thm:tighter_bounds}, as well as naive lower bound based on randomly sampling 10,000 perturbation vectors from an $\epsilon$-sized sphere. Figure \ref{fig:layers} shows the results for different layers of the MNIST network. Figure \ref{fig:full_network} shows the full-network local Lipschitz constants using the method discussed in Section \ref{sec:multiple_layers}. Table \ref{tab:computation_times} shows the computation times.   

\begin{figure}[ht]
\centering
\includegraphics[width=.24\textwidth]{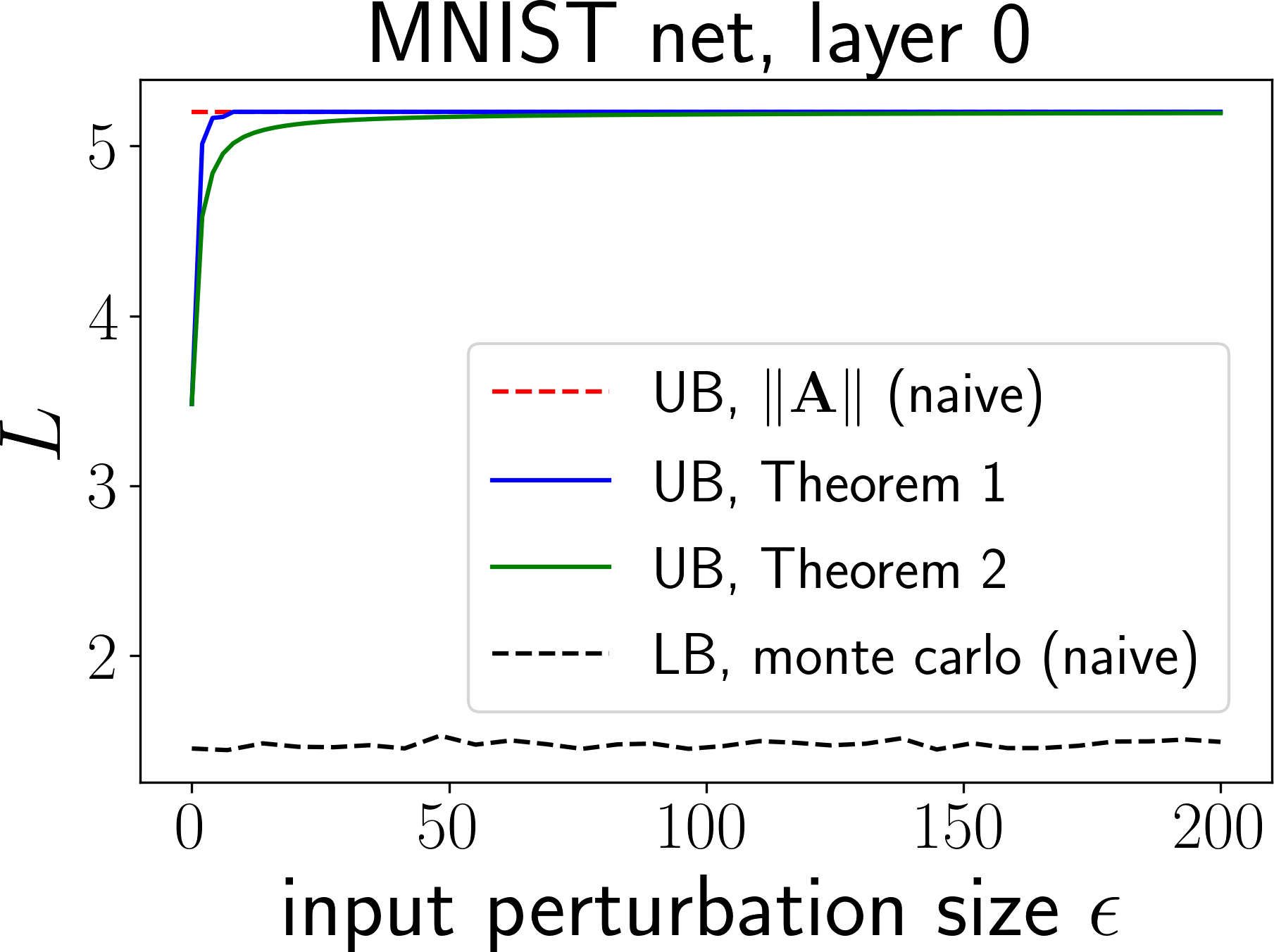}
\includegraphics[width=.24\textwidth]{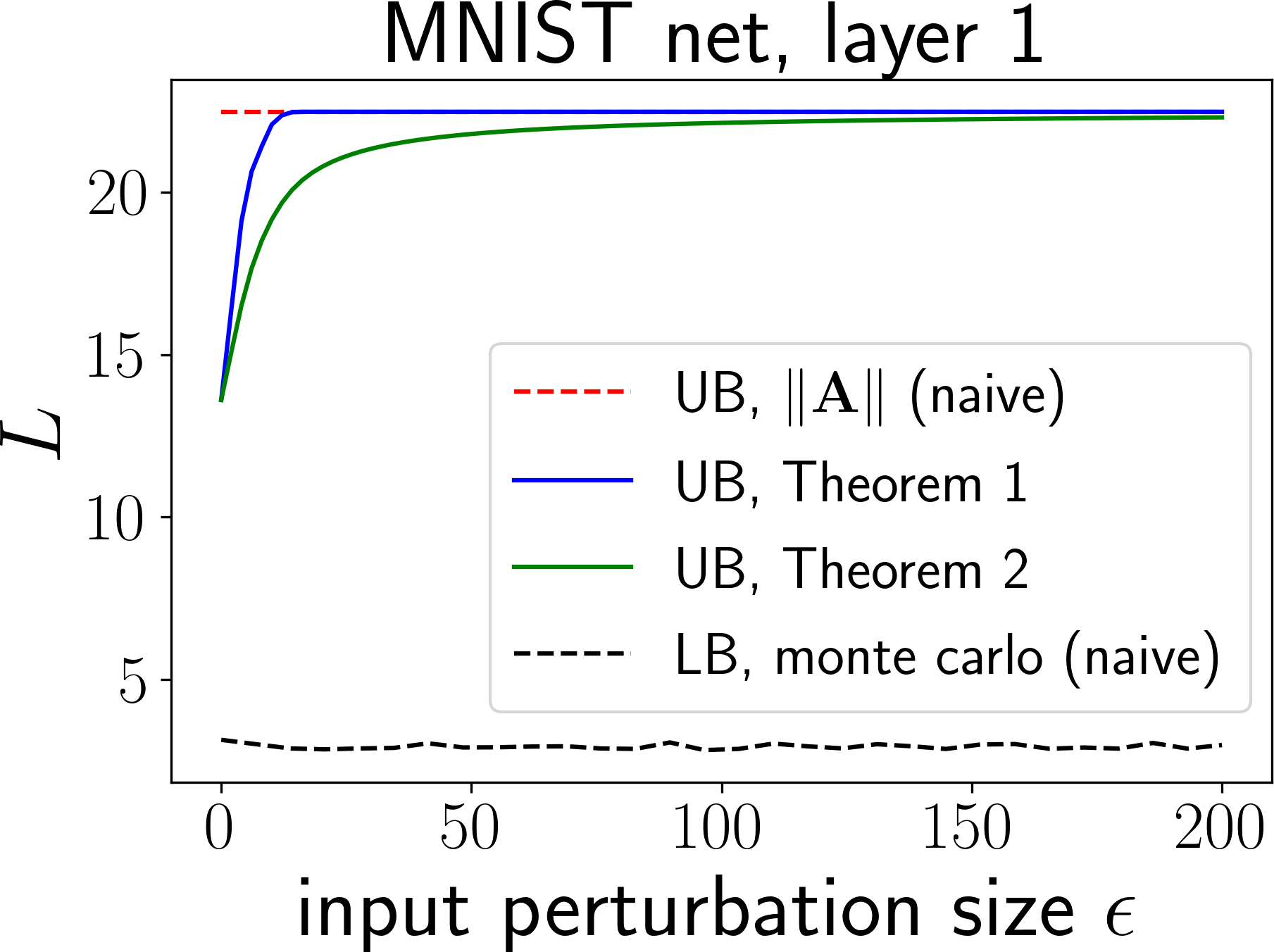}
\includegraphics[width=.24\textwidth]{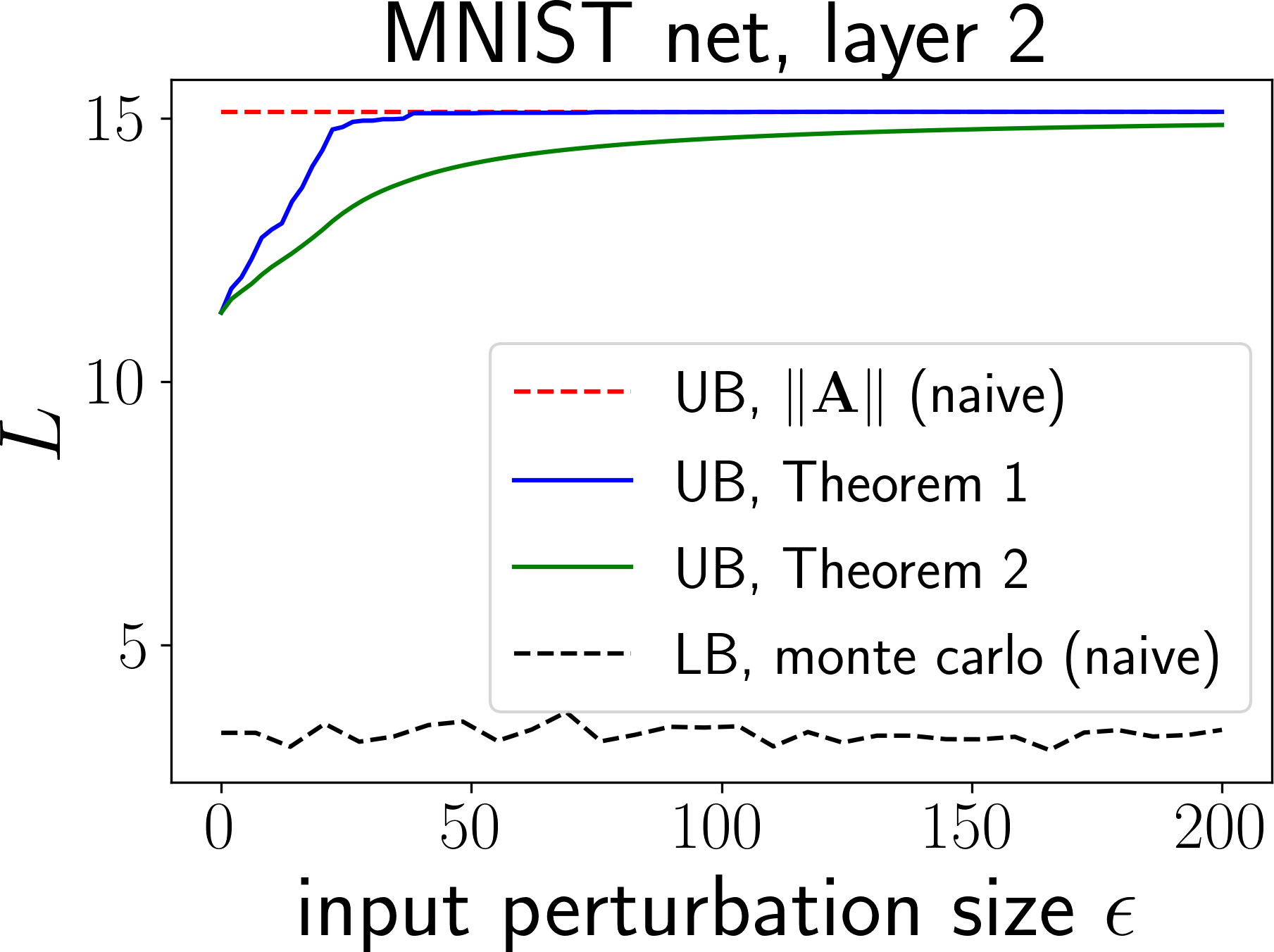}
\includegraphics[width=.24\textwidth]{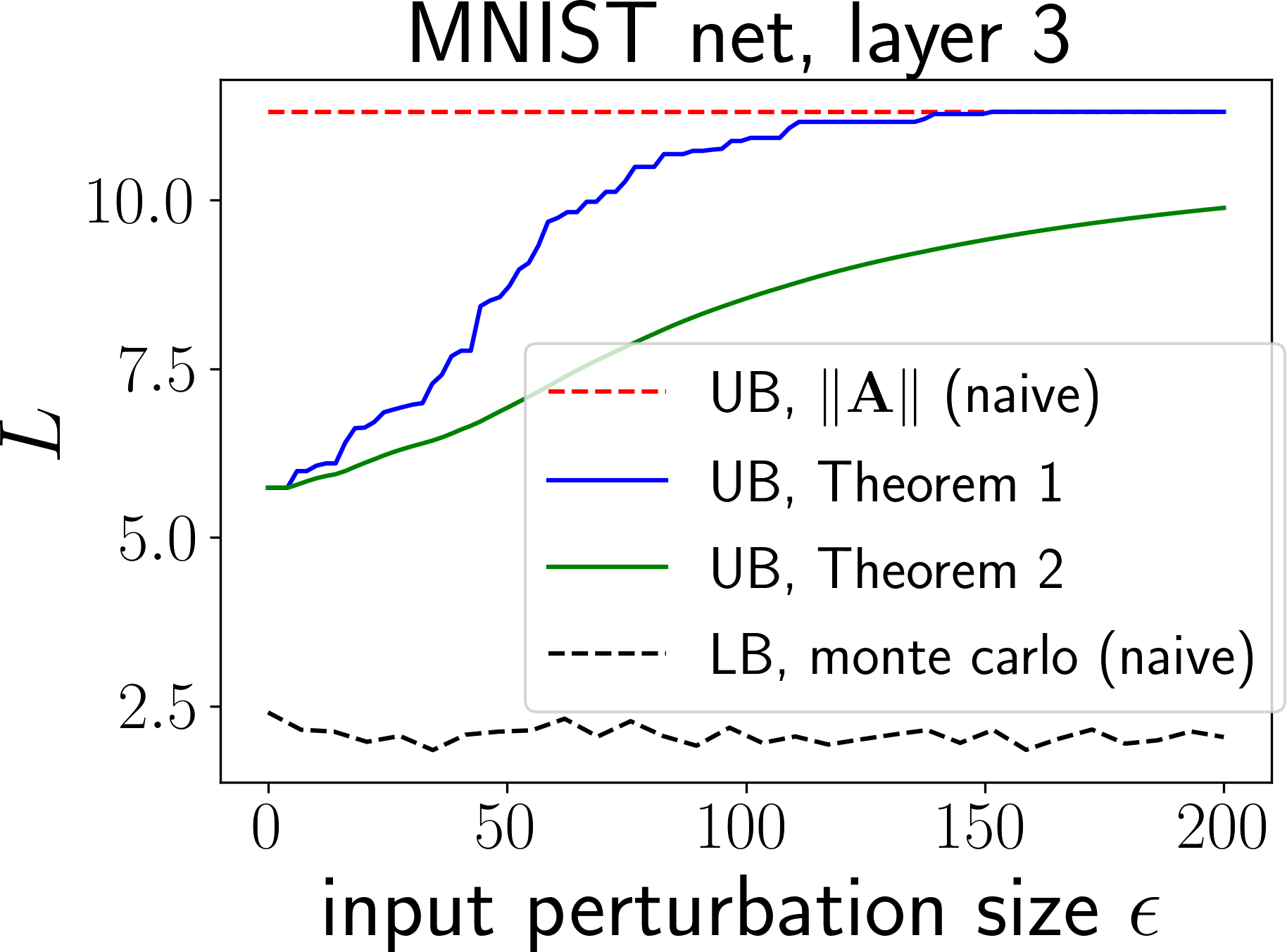}
\caption{Upper bounds (UB) and lower bounds (LB) on the local Lipschitz constants for the 4 affine-ReLU layers (convolution, convolution, fully-connected, fully-connected) of the MNIST net. Note that these results are computed with respect to a particular nominal input image.}
\label{fig:layers}
\end{figure}

\begin{figure}[ht]
\centering
\includegraphics[width=.32\textwidth]{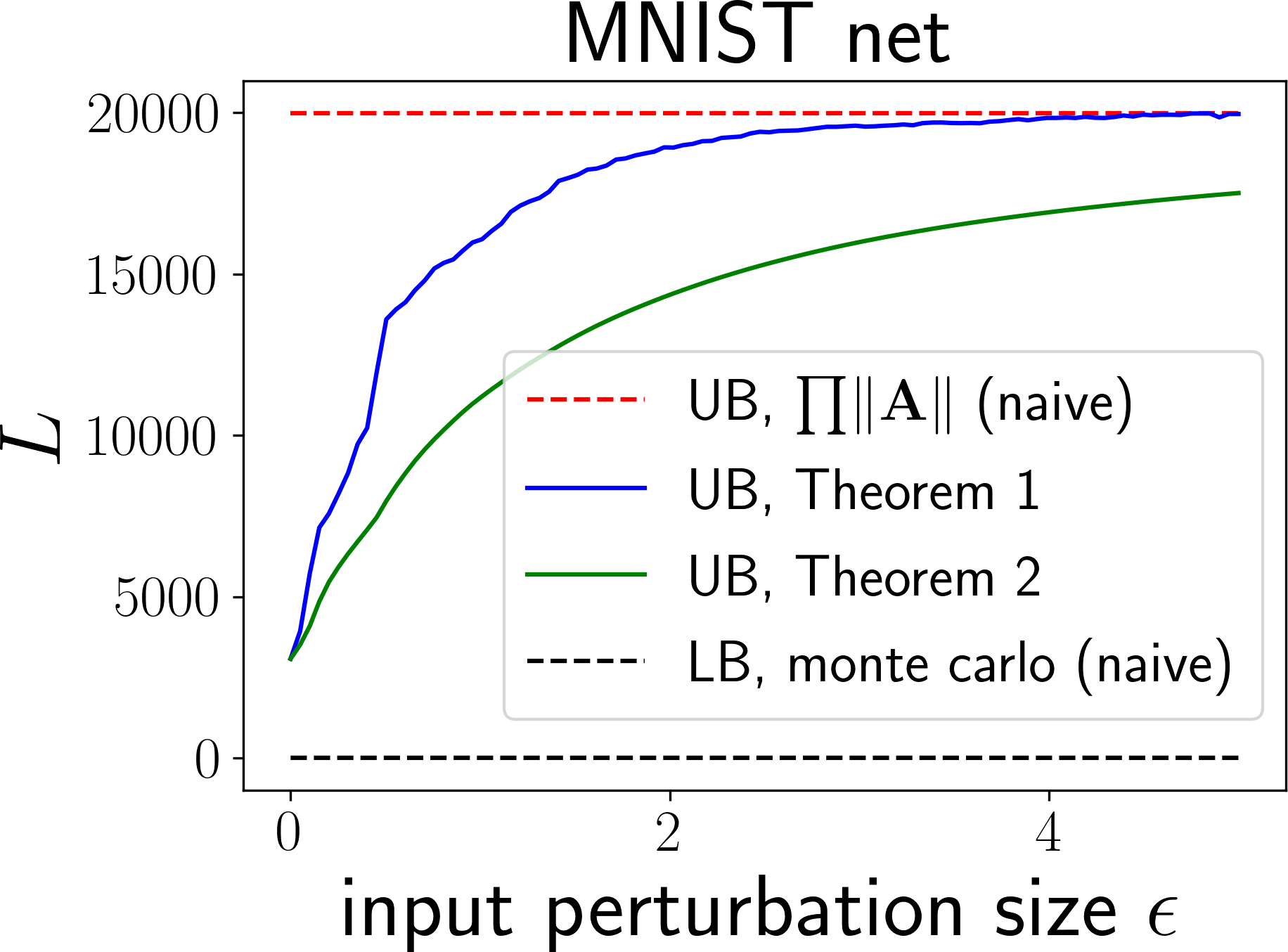}
\includegraphics[width=.32\textwidth]{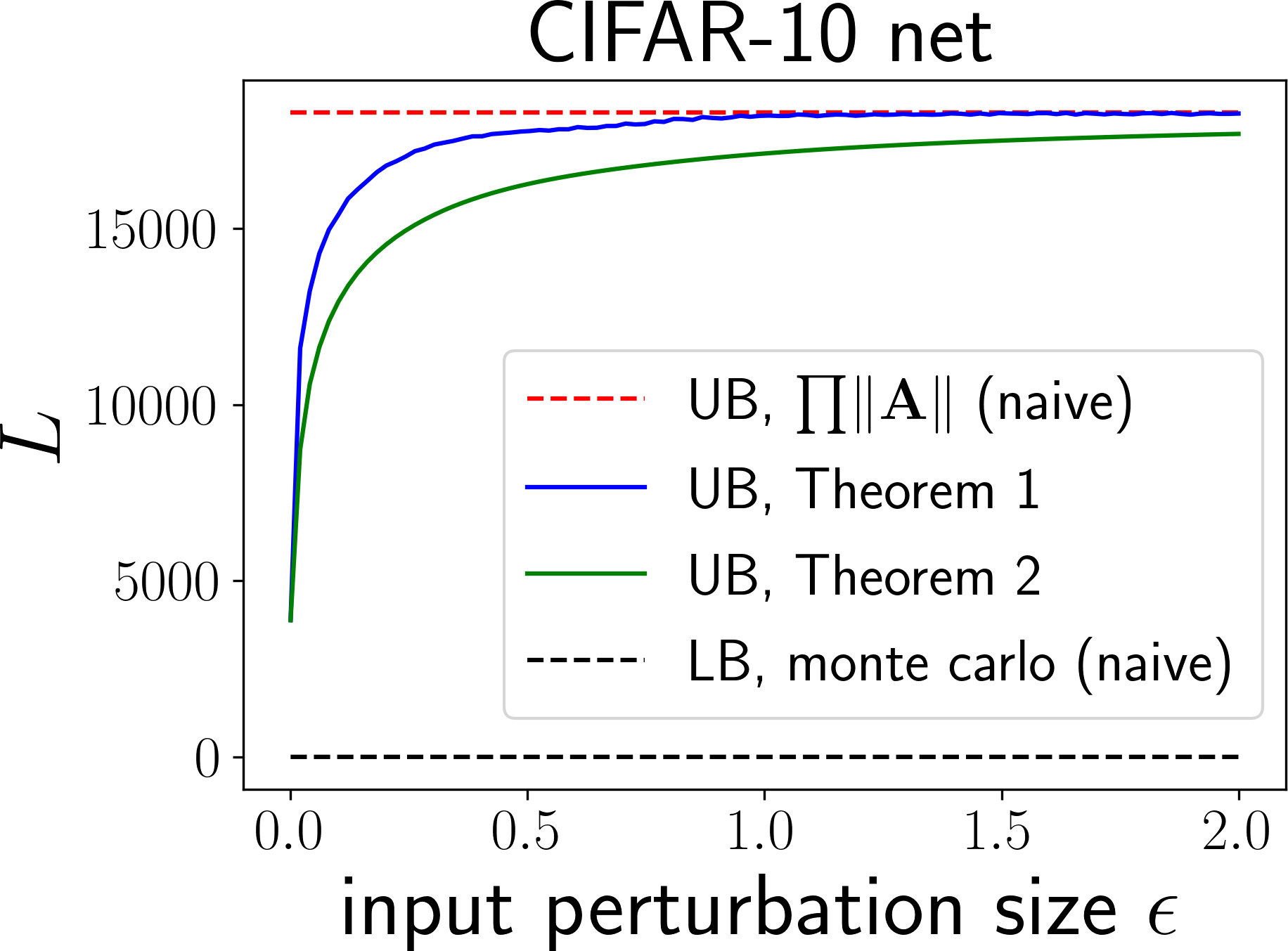}
\includegraphics[width=.32\textwidth]{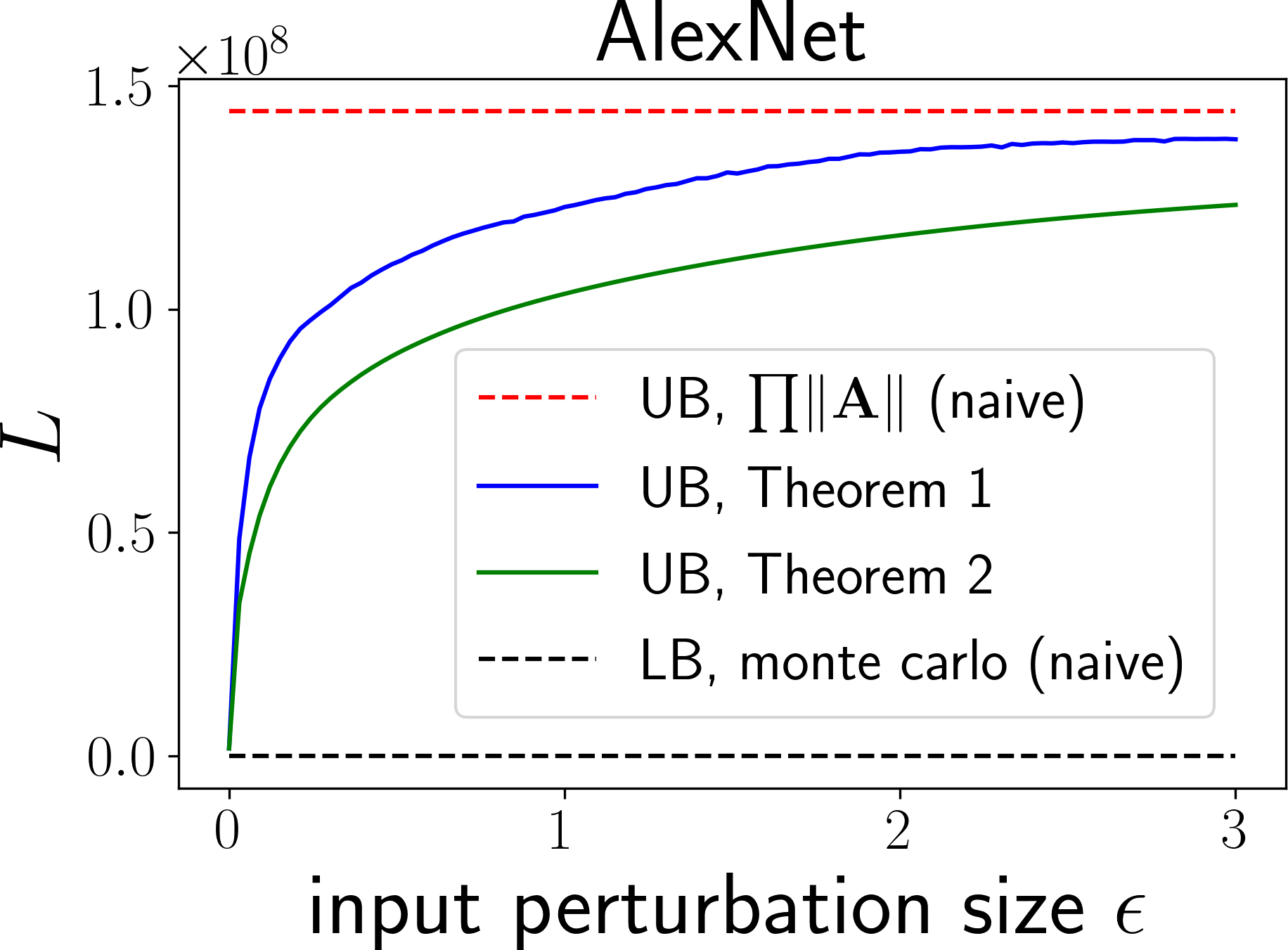}
\caption{Upper bounds (UB) and lower bounds (LB) on the local Lipschitz constants of MNIST, CIFAR-10, and AlexNet networks for various perturbation sizes. Note that these results are computed with respect to particular nominal input images.}
\label{fig:full_network}
\end{figure}

\begin{table}[ht]
\caption{Computation times for the local Lipschitz bounds.
%These times include a computation of the bounding region parameters $\boldsymbol{l}$ as well as the spectral norms of all matrices $\<R> \<A>$.
Computations were performed on a desktop computer with an Nvidia GTX 1080 Ti card.}
\label{tab:computation_times}
\centering
\begin{tabular}{ c c c c }
\toprule
\textit{network} & \textbf{MNIST net} & \textbf{CIFAR-10 net} & \textbf{AlexNet} \\
\midrule
\textit{computation time} & 2 sec & 58 sec & 72 min \\
%\textit{spectral norms} & 2 sec & 58 sec & 72 min
%\textit{total} & 2 sec & 58 sec & 72 min
\bottomrule
\end{tabular}
%\label{table:1}
\end{table}

The results show that our Lipschitz bounds increase with the size of the perturbation $\epsilon$, and approach the spectral norm of $\<A>$ for large $\epsilon$. For small perturbations, the bound is significantly lower than the naive bound.
%Note that there is a small amount of noise in the plots which is due to using the power iteration method to compute the spectral norms.

%Also note that the input image has 150,528 elements, with min and max element values of about -2 and 3. This means that the perturbations on the $x$-axis are relatively small, and correspond to only significantly changing several pixels.

%%%%%%%%%%%%%%%%%%%%%%%%%%%%%%%%%%%%%%%%%%%%%%%%%%%%%%%%%%%%%%%%%%%%%%%%%%%%%%%%
\section{Conclusion}

We have presented the idea of computing upper bounds on the local Lipschitz constant of an affine-ReLU function, which represents one layer of a neural network. We described how these bounds can be combined to determine the Lipschitz constant of a full network, and also how they can be computed in an efficient way, even for large networks. The results show that our bounds are tighter than the naive bounds for a full network, especially for small perturbations.
\iffalse
Some possible directions of future work are as follows.
%First, we note that the matrix form of convolution layers are sparse and has a cyclic structure, so we are interested in somehow exploiting that structure.
First, we could consider variations of the ReLU, or other activation functions. Second, our upper bound is computed with respect to one input, so it would be worthwhile to consider finding a bound that applies to all inputs in the input domain.
%Third, we believe that our upper bound is likely to be looser for larger perturbations.
Third, we are interested in applying a similar bounding approach to compute the local Lipschitz constants of other functions such as max pooling.
\fi

We believe that the most important direction of future work regarding our method is to more effectively apply it to multiple layers. While we can combine our layer-specific bounds as we have in Fig. \ref{fig:full_network}, it almost certainly leads to an overly conservative bound, especially for deeper networks. We also suspect that our bounds become more conservative for larger perturbations.
%We believe that future approaches may involve characterizing the singular values of the linear operators $\<A>$ may be one approach to help extend our work in this direction.
However, calculating tight Lipschitz bounds for large neural networks is still an open and challenging problem, and we believe our results provide a useful step forward.

%%%%%%%%%%%%%%%%%%%%%%%%%%%%%%%%%%%%%%%%%%%%%%%%%%%%%%%%%%%%%%%%%%%%%%%%%%%%%%%%
\section{Broader Impact}

We classify this work as basic mathematical analysis that applies to functions commonly used in neural networks. We believe this work could benefit those who are interested in developing more robust algorithms for safety-critical or other applications. It does not seem to us that this research puts anyone at a disadvantage. Additionally, since our bounds are provable, our method should not fail unless it is implemented incorrectly. Finally, we do not believe our method leverages any biases in data.

%%%%%%%%%%%%%%%%%%%%%%%%%%%%%%%%%%%%%%%%%%%%%%%%%%%%%%%%%%%%%%%%%%%%%%%%%%%%%%%%
\begin{ack}
This work was supported by ONR grant N000141712623.
\end{ack}

\section*{References}

\medskip

\small

% hide bibliography title
% see: https://tex.stackexchange.com/questions/22645/hiding-the-title-of-the-bibliography
\begingroup
\renewcommand{\section}[2]{}%
\bibliographystyle{plain}
\bibliography{root.bib}
\endgroup

%%%%%%%%%%%%%%%%%%%%%%%%%%%%%%%%%%%%%%%%%%%%%%%%%%%%%%%%%%%%%%%%%%%%%%%%%%%%%%%%
%%%%%%%%%%%%%%%%%%%%%%%%%%%%%%%%%%%%%%%%%%%%%%%%%%%%%%%%%%%%%%%%%%%%%%%%%%%%%%%%
%%%%%%%%%%%%%%%%%%%%%%%%%%%%%%%%%%%%%%%%%%%%%%%%%%%%%%%%%%%%%%%%%%%%%%%%%%%%%%%%
\newpage

%\renewcommand\thesection{\ifcase\value{section}T\fi}
%\section{Test secccc}

\appendix

\section{Appendix}

\subsection{Proofs} \label{sec:appendix_proofs}

%%%%%%%%%%%%%%%%%%%%%%%%%%%%%%%%%%%%%%%%%%%%%%%%%%%%%%%%%%%%%%%%%%%%%%%%%%%%%%%%
% Lemma 1
\begin{customlem}{1} \label{lem:basic_bound}
Consider the affine function $\<A> \<x> + \<b>$, its domain $\mathcal{X}$, and the piecewise representation of the affine-ReLU function in \textup{(4)}. We have the following upper bound on the affine-ReLU function's local Lipschitz constant: $L( \<x>_0, \mathcal{X} ) \leq \max_{\<x> \in \mathcal{X}} \sum_{i=1}^p \Delta \alpha_i \norm{\<R>_i \<A>}$.
\end{customlem}
\begin{proof}
We can start with \textup{(7)} and plug in \textup{(4)}:
\begin{align*}
L \left( \mathcal{X}, \<x>_0 \right)
&= \max_{\<x> \in \mathcal{X}} \frac{\norm{\<relu>(\<A> \<x> + \<b>) - \<relu>(\<b>)}}{\norm{\<x>}} \\
&= \max_{\<x> \in \mathcal{X}} \frac{\norm{( \<R>_{\<b>} \<b> + \sum_{i=1}^p \Delta \alpha_i \<R>_i \<A> \<x> ) - \<R>_{\<b>} \<b>}}{\norm{\<x>}} \\
&= \max_{\<x> \in \mathcal{X}} \frac{\norm{\sum_{i=1}^p \Delta \alpha_i \<R>_i \<A> \<x>}}{\norm{\<x>}} \\
&\leq \max_{\<x> \in \mathcal{X}} \frac{\sum_{i=1}^p \Delta \alpha_i \norm{\<R>_i \<A> \<x>}}{\norm{\<x>}} \\
&\leq \max_{\<x> \in \mathcal{X}} \frac{\sum_{i=1}^p \Delta \alpha_i \norm{\<R>_i \<A>} \norm{\<x>}}{\norm{\<x>}} \\
&= \max_{\<x> \in \mathcal{X}} \sum_{i=1}^p \Delta \alpha_i \norm{\<R>_i \<A>} .
\end{align*}
\end{proof}

\hrulefill

%%%%%%%%%%%%%%%%%%%%%%%%%%%%%%%%%%%%%%%%%%%%%%%%%%%%%%%%%%%%%%%%%%%%%%%%%%%%%%%%
% Proposition 1
\begin{customprp}{1} \label{prp:naive_upper_bound}
Consider the affine function $\<A> \<x> + \<b>$ and its domain $\mathcal{X}$. The spectral norm of $\<A>$ is an upper bound on the affine-ReLU function's local Lipschitz constant: $L( \<x>_0, \mathcal{X} ) \leq \norm{\<A>}$.
\end{customprp}
\begin{proof}
Consider the inequality from Lemma \textup{1}: $L( \<x>_0, \mathcal{X} ) \leq \max_{\<x> \in \mathcal{X}} \sum_{i=1}^p \Delta \alpha_i \norm{\<R>_i \<A>}$. Note that the ReLU matrix $\<R>_i$ is a diagonal matrix of 0's and 1's, so for any $\<v> \in \mathbb{R}^n$, the non-zero elements of $\<R>_i \<A> \<v>$ will be a subset of the non-zero elements of $\<A> \<v>$. Therefore, $\norm{\<A>} \geq \norm{\<R>_i \<A>}$ for all $\<R>_i$ and all $\<x> \in \mathcal{X}$, and we can rearrange the inequality from Lemma \textup{1} as follows:
\begin{align*}
L( \<x>_0, \mathcal{X} ) &\leq \max_{\<x> \in \mathcal{X}} \sum_{i=1}^p \Delta \alpha_i \norm{\<R>_i \<A>} \\
&\leq \max_{\<x> \in \mathcal{X}} \sum_{i=1}^p \Delta \alpha_i \norm{\<A>} \\
&= \norm{\<A>} .
\end{align*}
Where in the last step we have used the fact that $\sum_{i=1}^p \Delta \alpha_i = 1$.
\end{proof}

\hrulefill

%%%%%%%%%%%%%%%%%%%%%%%%%%%%%%%%%%%%%%%%%%%%%%%%%%%%%%%%%%%%%%%%%%%%%%%%%%%%%%%%
% Lemma 2
\begin{customlem}{2}
Consider a bounding region $\mathcal{H}$ and its bounding ReLU matrix $\overline{\<R>}$. For any two points $\<y>_a, \<y>_b \in \mathcal{H}$, the following inequality holds: $\norm{\overline{\<R>} (\<y>_b - \<y>_a)} \geq \norm{\<R>_{\<y>_b} \<y>_b - \<R>_{\<y>_a} \<y>_a}$.
\label{lem:bounding_region_vectors}
\end{customlem}
\begin{proof}
Since the $\<R>$ matrices are diagonal, we can consider this problem on an element-by-element basis, and show that each element of $\overline{\<R>} (\<y>_b - \<y>_a)$ is greater in magnitude than its counterpart in $\<R>_{\<y>_b} \<y>_b - \<R>_{\<y>_a} \<y>_a$. For a given $i$, let $y_a$ and $y_b$ denote the $i^{th}$ entry of $\<y>_a$ and $\<y>_b$, respectively, and let $R_a$, $R_b$, and $\overline{R}$ denote the $(i,i)^{th}$ entries of $\<R>_{\<y>_a}$, $\<R>_{\<y>_b}$, and $\overline{\<R>}$, respectively. We can write the $i^{th}$ element of $\overline{\<R>} (\<y>_b - \<y>_a)$ as $\overline{R}(y_b-y_a)$ and the corresponding element of $\<R>_{\<y>_b} \<y>_b - \<R>_{\<y>_a} \<y>_a$ as $R_b y_b - R_a y_a$.

Since $R_{y_a}=1$ or $R_{y_a}=1$ imply $\overline{R}=1$, there are five possible cases we have to consider, which are shown in the table below.
\begin{center}
\begin{tabular}{ c c c c c c }
\toprule
& $R_a$ & $R_b$ & $\overline{R}$ & $\abs{\overline{R} (y_b - y_a)}$ & $\abs{R_b y_b - R_a y_a}$ \\ 
\midrule
\textit{case 1} & 0 & 0 & 0 & 0 & 0 \\
\textit{case 2} & 0 & 0 & 1 & $\abs{y_b - y_a}$ & 0 \\
\textit{case 3} & 1 & 0 & 1 & $\abs{y_b - y_a}$ & $\abs{y_a}$ \\
\textit{case 4} & 0 & 1 & 1 & $\abs{y_b - y_a}$ & $\abs{y_b}$ \\
\textit{case 5} & 1 & 1 & 1 & $\abs{y_b - y_a}$ & $\abs{y_b -y_a}$ \\
\bottomrule
\end{tabular}
\end{center}
For cases 1, 2, and 5 it is clear that $\abs{\overline{R}(y_b - y_a)} \geq \abs{R_b y_b - R_a y_a}$. For case 3, we note that if $R_a=1$ and $R_b=0$, then $y_a \geq 0$ and $y_b \leq 0$, which means that $y_b - y_a$ is a non-positive number that has magnitude equal to or less than the magnitude of $y_a$. This implies that $\abs{\overline{R}(y_b - y_a)} \geq \abs{R_b y_b - R_a y_a}$ for case 3. Similar logic can be applied to case 4, except in this case $y_b - y_a$ is a non-negative number that has magnitude equal to or greater than the magnitude of $y_b$. This implies that $\abs{\overline{R}(y_b - y_a)} \geq \abs{R_b y_b - R_a y_a}$ for case 4.

We showed that for all $i$, each element of $\overline{\<R>} (\<y>_b - \<y>_a)$ is greater in magnitude than the corresponding element in $\<R>_{\<y>_b} \<y>_b - \<R>_{\<y>_a} \<y>_a$, i.e. $\abs{\overline{R} (y_b - y_a)} \geq \abs{R_b y_b - R_a y_a}$. This implies $\norm{\overline{\<R>} (\<y>_b - \<y>_a)} \geq \norm{\<R>_{\<y>_b} \<y>_b - \<R>_{\<y>_a} \<y>_a}$.
\end{proof}

\hrulefill

%%%%%%%%%%%%%%%%%%%%%%%%%%%%%%%%%%%%%%%%%%%%%%%%%%%%%%%%%%%%%%%%%%%%%%%%%%%%%%%%
% Theorem 2
\begin{customthm}{2} \label{thm:tighter_bounds}
Consider the affine function $\<A> \<x> + \<b>$ and its domain $\mathcal{X}$. Consider the nested bounding regions $\beta_i \mathcal{H}$, their scale factors $\Delta \beta_i$ and their bounding ReLU matrices $\overline{\<R>}_i$ as described in Section 4.2. The following is an upper bound on the affine-ReLU function's local Lipschitz constant: $L( \<x>_0, \mathcal{X} ) \leq \sum_{i=1}^q \Delta \beta_i \norm{\overline{\<R>}_i \<A>}$.
\end{customthm}
\begin{proof}
First, define the affine transformation of $\beta_i \<x>$ to be $\<y>_i = \beta_i \<A> \<x> + \<b>$. We can write the function $\<relu>(\<A> \<x> + \<b>)$ as the sum of the differences of the function taken across each segment. For the segment from $i{-}1$ to $i$, the difference in the affine-ReLU function is $\<R>_{\<y>_i} \<y>_i - \<R>_{\<y>_{i-1}} \<y>_{i-1}$. So we can write the total function as
\begin{align*}
\<relu>(\<A> \<x> + \<b>)
&= (\<R>_{\<b>} - \<0>) + \sum_{i=1}^q \left( \<R>_i \<y>_i - \<R>_{i-1} \<y>_{i-1} \right) \\
&= \<R>_{\<b>} + \sum_{i=1}^q \left( \<R>_i \<y>_i - \<R>_{i-1} \<y>_{i-1} \right) .
\end{align*}
Plugging the equation above into \textup{(7)}, we can write the local Lipschitz constant as
\begin{align*}
L( \<x>_0, \mathcal{X} ) 
&= \max_{\<x> \in \mathcal{X}} \frac{\norm{\<relu>(\<A> \<x> + \<b>) - \<relu>(\<b>) }}{\norm{\<x>}} \\
&= \max_{\<x> \in \mathcal{X}} \frac{\norm{\<R>_{\<b>} \<b> + \sum_{i=1}^q ( \<R>_{\<y>_i} \<y>_i - \<R>_{\<y>_{i-1}} \<y>_{i-1} ) - \<R>_{\<b>} \<b>}}{\norm{\<x>}} \\
&= \max_{\<x> \in \mathcal{X}} \frac{\norm{\sum_{i=1}^q ( \<R>_{\<y>_i} \<y>_i - \<R>_{\<y>_{i-1}} \<y>_{i-1} )}}{\norm{\<x>}} \\
&\leq \max_{\<x> \in \mathcal{X}} \frac{\sum_{i=1}^q \norm{ ( \<R>_{\<y>_i} \<y>_i - \<R>_{\<y>_{i-1}} \<y>_{i-1} )}}{\norm{\<x>}} .
\end{align*}

Recalling that $\<y>_{i-1} = \beta_{i-1} \<A> \<x> + \<b>$ and $\<y>_i = \beta_i \<A> \<x> + \<b>$, and that $\beta_{i-1} < \beta_i$, we know that $\<y>_{i-1},\<y>_i \in \beta_i \mathcal{H}$. So, using Lemma 2 we can rearrange the equation above as follows:
\begin{align*}
L( \<x>_0, \mathcal{X} )
&\leq \max_{\<x> \in \mathcal{X}} \frac{\sum_{i=1}^q \norm{\overline{\<R>}_i (\<y>_i - \<y>_{i-1})}}{\norm{\<x>}} \\
&= \max_{\<x> \in \mathcal{X}} \frac{\sum_{i=1}^q \norm{\overline{\<R>}_i \left( \beta_i \<A> \<x>  + \<b> - (\beta_{i-1} \<A> \<x> + \<b>) \right)}}{\norm{\<x>}} \\
&= \max_{\<x> \in \mathcal{X}} \frac{\sum_{i=1}^q \Delta \beta_i \norm{\overline{\<R>}_i \<A> \<x>}}{\norm{\<x>}} \\
&\leq \frac{\sum_{i=1}^q \Delta \beta_i \norm{\overline{\<R>}_i \<A>} \norm{\<x>}}{\norm{\<x>}} \\
&= \sum_{i=1}^q \Delta \beta_i \norm{\overline{\<R>}_i \<A>} .
\end{align*}
\end{proof}

%%%%%%%%%%%%%%%%%%%%%%%%%%%%%%%%%%%%%%%%%%%%%%%%%%%%%%%%%%%%%%%%%%%%%%%%%%%%%%%%
%%%%%%%%%%%%%%%%%%%%%%%%%%%%%%%%%%%%%%%%%%%%%%%%%%%%%%%%%%%%%%%%%%%%%%%%%%%%%%%%
%%%%%%%%%%%%%%%%%%%%%%%%%%%%%%%%%%%%%%%%%%%%%%%%%%%%%%%%%%%%%%%%%%%%%%%%%%%%%%%%
\subsection{Bounding region determination for various domains} \label{sec:appendix_bounding_region}

We have presented the idea of considering an affine function $\<A> \<x> + \<b>$ with domain $\mathcal{X}$ and range $\mathcal{Y}$. We are interested in determining the axis-aligned bounding region $\mathcal{H}$ around $\mathcal{Y}$. We will now show how to tightly compute this region for various domains. Recall from Section 4.1 that we denote the upper bounding vertex of the region as $\overline{\<y>}$, and define $\boldsymbol{l} \coloneqq \overline{\<y>} - \<b>$ (the distance from $\<b>$ to $\overline{\<y>}$). We let $\<a>_1^T,...,\<a>_m^T \in \mathbb{R}^n$ denote the rows of $\<A>$, and $a_{ij}$ denote the $j^{th}$ element of $\<a>_i$. Similarly, we let $\overline{y}_i$ and $l_i$ denote the $i^{th}$ elements of $\overline{\<y>}$ and $\boldsymbol{l}$, respectively. Denoting $\<e>_i \in \mathbb{R}^m$ as the $i^{th}$ standard basis vector, we can write this problem as
\begin{align*}
\overline{y}_i &= \max_{\<x> \in \mathcal{X}} ~ \<e>_i^T (\<A> \<x> + \<b>) \\
&= \max_{\<x> \in \mathcal{X}} ~ \<a>_i^T \<x> + b_i .
\end{align*}
We can subtract out the bias term in the maximization above since it is constant. By doing so, our maximization will find $l_i$ instead of $\overline{y}_i$. We also define $\<x>^{*,i} \in \mathbb{R}^n$ as the maximizing vector:
\begin{align*}
l_i &= \max_{\<x> \in \mathcal{X}} ~ \<a>_i^T \<x> \\
\<x>^{*,i} &= \argmax_{\<x> \in \mathcal{X}} ~ \<a>_i^T \<x> .
\end{align*}

\subsubsection*{Domain 1: $\mathcal{X} = \{ \<x> ~~ | ~~ \norm{\<x>}_1 \leq \epsilon \}$}

Let $x_j$ denote the $j^{th}$ element of $\<x>$. In this case we have$\sum_j \abs{x_j} \leq \epsilon$. The quantity $\<a>_i^T \<x>$ will be maximized when the element of $\<a>_i$ with the largest magnitude is given all of the weight. In other words,
\begin{align*}
j^* &\coloneqq \argmax_j \abs{a_{ij}} \\
x_j^{*,i} &= 
\begin{cases}
\epsilon \cdot \sgn(a_{ij^*}), &j = j^* \\
0, &\text{otherwise}
\end{cases}
\\
l_i &= \abs{a_{ij^*}} .
\end{align*}

\subsubsection*{Domain 2: $\mathcal{X} = \{ \<x> ~~ | ~~ \norm{\<x>}_2 \leq \epsilon \}$}
In this case, we are maximizing over all vectors $\<x>$ with length less than or equal to $\epsilon$. So, the maximum will occur when $\<x>$ points in the direction of $\<a>_i$ and has the largest possible magnitude (i.e. $\epsilon$). Note that intuitively this can be thought of as maximizing the dot product of an $\epsilon$-sized $n$-sphere with $\<a>_i$. We have
\begin{align*}
\<x>^{*,i} &= \epsilon \frac{\<a>_i}{\norm{\<a>_i}_2} \\
l_i &= \<a>_i^T \left( \epsilon \frac{\<a>_i}{\norm{\<a>_i}_2} \right) \\
&= \epsilon \norm{\<a>_i}_2 .
\end{align*}
Note that we have assumed that $\<a>_i \neq \<0>$. If $\<a>_i = \<0>$, then it is obvious that any $\<x> \in \mathcal{X}$ will produce a maximum value of $l_i = 0$, and the last equation still holds.

\subsubsection*{Domain 3: $\mathcal{X} = \{ \<x> ~~ | ~~ \norm{\<x>}_{\infty} \leq \epsilon \}$}
In this case, $\<x> \in [-\epsilon,\epsilon]^n$. So, the quantity $\<a>_i^T \<x>$ is maximized when $x_j = \epsilon$ for positive $a_{ij}$, and $x_j = -\epsilon$ for negative $a_{ij}$. So, we have
\begin{align*}
x_j^{*,i} &= 
\begin{cases}
-\epsilon, &a_{ij} < 0 \\
\epsilon, &a_{ij} > 0 \\
0, &a_{ij} = 0
\end{cases}
\\
l_i &= \sum_j \abs{a_{ij}} .
\end{align*}
Note that when $a_{ij}=0$, the value of $x_j^{*,i}$ does not matter as long as it is in the range $[-\epsilon,\epsilon]$. But we define it as zero so that when we consider non-negative domains in the next section, we can simply replace the matrix $\<A>$ with its positive part $\<A>^+$.

\subsubsection*{Non-negative Domains: $\mathcal{X} = \{ \<x> ~~ | ~~ \norm{\<x>}_q \leq \epsilon, ~~ \<x> \geq \<0> \}$}

In many cases, due to the affine-ReLU function being preceded by a ReLU, the domain $\mathcal{X}$ will consist of vectors with non-negative entries. In these cases, the bounding region often becomes smaller (i.e. some or all elements of $\overline{\<y>}$ are smaller). For the 1, 2, or $\infty$-norms above, we can first decompose the $\<A>$ matrix into its positive and negative parts: $\<A> = \<A>^+ - \<A>^-$. Since placing emphasis on a negative element of $\<a>_i$ will always be suboptimal, we can apply the same analysis in Domains 1, 2, and 3, except replace $\<A>$ with $\<A>^+$.

\subsubsection*{Efficient computation}

Note that for large convolutional layers, it is too expensive to represent the entire matrix $\<A>$. In these cases, we can obtain the $i^{th}$ row of $\<A>$ using the transposed convolution operator. More specifically, we create a transposed convolution function with no bias, based on the original convolution function. Then, by noting that if we consider a standard basis vector $\<e>_i \in \mathbb{R}^m$, the $i^{th}$ column of $\<A>^T$ (and $i^{th}$ row of $\<A>$) is given by $\<A>^T \<e>_i$. Therefore, by plugging in the $i^{th}$ standard basis vector to the transposed convolution function, we can obtain $\<a>_i$, the $i^{th}$ row of $\<A>$. Note that a vector $\<e>_i$ must first be reshaped into the proper input dimension before plugging it into the transposed convolution function. Furthermore, to reduce computation time in practice, instead of plugging in each standard basis vector $\<e>_i$ one at a time, we plug in a batch of different standard basis vectors to obtain multiple rows of $\<A>$.

%%%%%%%%%%%%%%%%%%%%%%%%%%%%%%%%%%%%%%%%%%%%%%%%%%%%%%%%%%%%%%%%%%%%%%%%%%%%%%%%
%%%%%%%%%%%%%%%%%%%%%%%%%%%%%%%%%%%%%%%%%%%%%%%%%%%%%%%%%%%%%%%%%%%%%%%%%%%%%%%%
%%%%%%%%%%%%%%%%%%%%%%%%%%%%%%%%%%%%%%%%%%%%%%%%%%%%%%%%%%%%%%%%%%%%%%%%%%%%%%%%
\subsection{Neural network architectures} \label{sec:appendix_architectures}

We used three neural networks in this paper. The first network is based on the MNIST dataset and we refer to it as ``MNIST net''. We constructed MNIST net ourselves and trained it to 99\% top-1 test accuracy in 100 epochs. The second network is based on the CIFAR-10 dataset and we refer to it as ``CIFAR-10 net''. We constructed CIFAR-10 net ourselves and trained it to 84\% top-1 test accuracy in 500 epochs. The third network is the pre-trained implementation of AlexNet from Pytorch's \texttt{torchvision} package. The following table shows the architectures of MNIST net and CIFAR-10 net.

\begin{table}[ht]
\caption{Networks we constructed for this paper. Convolution layers are denoted as conv\{\textit{kernel size}\}-\{\textit{output channels}\}. Max pooling layers are denoted as maxpool\{\textit{kernel size}\}, and fully-connected layers are denoted as FC-\{\textit{output features}\}. All convolution layers are followed by a ReLU and have a stride of 1. All fully-connected layers are followed by a ReLU unless it is the last layer.}
\centering
\begin{tabular}{ c c }
\toprule
\textbf{MNIST net} & \textbf{CIFAR-10 net} \\
\midrule
conv5-6 & conv3-32 \\
maxpool2 & conv3-32 \\
conv5-16 & maxpool2 \\
maxpool2 & dropout \\
FC-120 & conv3-64 \\
FC-84 & conv3-64 \\
FC-10 & maxpool2 \\
 & dropout \\
 & FC-512 \\
 & dropout \\
 & FC-10 \\
\bottomrule
\end{tabular}
\label{tab:network_architectures}
\end{table}

\end{document}